\documentclass[12pt,onecolumn,journal]{IEEEtran}

\ifCLASSINFOpdf

\else

\fi
\usepackage[normalem]{ulem}
\usepackage{cite}
\usepackage{amsmath}
\usepackage{amssymb}
\usepackage{amsfonts}
\usepackage{amsthm}
\usepackage{mathtools}
\usepackage{graphicx}
\usepackage{color}
\usepackage{epstopdf}
\usepackage{caption}
\usepackage{multirow}
\usepackage{subfigure}
\usepackage[numbers,sort&compress]{natbib}
\usepackage[left=2cm, right=2cm, top=2.54cm]{geometry}
\DeclareMathOperator{\Var}{\operatorname{Var}}

\newcommand{\norm}[1]{\left\lVert#1\right\rVert}
\newcommand{\stkout}[1]{\ifmmode\text{\sout{\ensuremath{#1}}}\else\sout{#1}\fi}

\usepackage{xcolor,cancel}

\newtheorem{thm}{Theorem}
\newtheorem{ass}{Assumption}
\newtheorem{lem}{Lemma}

\newtheorem{prop}{Proposition}

\begin{document}
	
	\title{Analysis of KNN Density Estimation}
	
	\author{Puning Zhao and
		Lifeng Lai\thanks{Puning Zhao and Lifeng Lai are with Department of Electrical and Computer Engineering, University of California, Davis, CA, 95616. Email: \{pnzhao,lflai\}@ucdavis.edu. This work was supported by the National Science Foundation under grants CCF-17-17943, ECCS-17-11468 and CNS-18-24553. }
	}
	\maketitle
	
	\begin{abstract}
	We analyze the $\ell_1$ and $\ell_\infty$ convergence rates of $k$ nearest neighbor density estimation method. Our analysis includes two different cases depending on whether the support set is bounded or not. In the first case, the probability density function has a bounded support and is bounded away from zero. We show that kNN density estimation is minimax optimal under both $\ell_1$ and $\ell_\infty$ criteria, if the support set is known. If the support set is unknown, then the convergence rate of $\ell_1$ error is not affected, while $\ell_\infty$ error does not converge. In the second case, the probability density function can approach zero and is smooth everywhere. Moreover, the Hessian is assumed to decay with the density values. For this case, our result shows that the $\ell_\infty$ error of kNN density estimation is nearly minimax optimal. The $\ell_1$ error does not reach the minimax lower bound, but is better than kernel density estimation. 
	\end{abstract}
\begin{IEEEkeywords}
	Density estimation, KNN, Functional approximation
\end{IEEEkeywords}

\section{Introduction}

Nonparametric density estimation, whose goal is to estimate the probability density function (pdf) based on a finite set of identically and independently distributed (i.i.d) samples, is widely used in statistics and machine learning. For example, nonparametric density estimation can be used in mode estimation \cite{parzen1962estimation}, nonparametric classification \cite{rigollet2007generalization,chaudhuri2008classification}, Monte Carlo computational methods \cite{smith2013sequential}, and clustering \cite{chaudhuri2010rates,jiang2017modal,rinaldo2010generalized}, etc. Common methods for the nonparametric density estimation include histogram method, kernel method and $k$-Nearest Neighbor (kNN) method \cite{silverman2018density, devroye1977strong,bhattacharya1987weak,ouadah2013uniform}, etc. Among these approaches, the kernel and kNN methods are popular ones. The kernel method \cite{rosenblatt1956remarks,parzen1962estimation} estimates the density by calculating the convolution of the empirical distribution with a symmetric and normalized kernel function. The kNN method \cite{loftsgaarden1965nonparametric} estimates the density value at point $\mathbf{x}$ based on the distance between $\mathbf{x}$ and its $k$-th nearest neighbor. A large kNN distance indicates that the density is usually small, and vice versa. Compared with other methods, the kNN density estimation method has several advantages. It is purely nonparametric and hence can flexibly adapt to any underlying pdf, as long as the pdf is continuous. Moreover, the kNN method is convenient to use and has desirable time complexity. The parameter tuning is simple since the only parameter we need to adjust is $k$. 

Depending on the purpose of the density estimation, we may use different criteria to evaluate an estimator's performance. In some applications, we use the uniform bound, i.e. $||\hat{f}-f||_\infty$, in which $f$ is the real pdf and $\hat{f}$ is the estimated pdf. For example, if we hope to find the mode, which is the point with maximum pdf \cite{parzen1962estimation}, then the accuracy guarantee relies heavily on the supremum estimation error. For other purposes, such as nonparametric classification and bootstrapping, it may be better to consider the estimation error in the whole domain, instead of only considering its supremum value. For example, in nonparametric classification with Bayes rule, the excess risk of classification can be bounded with the $\ell_1$ error of the density estimation. The uniform and $\ell_1$ convergence properties of the kernel density estimation method have already been discussed in many previous literatures, see \cite{gine2002rates,einmahl2005uniform,jiang2017uniform,devroye1979l_1,kim2018uniform} and references therein. However, the understanding of the convergence properties of the kNN method is less complete, and still needs further analysis. In \cite{biau2015lectures}, it was shown that the kNN method is uniformly consistent if the pdf is smooth everywhere. However, the convergence rate is still unknown. \cite{mack1983rate} derived the uniform convergence rate of the kNN density estimate for one dimensional distributions, under the condition that the density is bounded away from zero and the support is a continuous closed interval. The analysis in \cite{mack1983rate} is not suitable for other commonly seen pdfs, especially for those with high dimensions, and those with unbounded supports such as Gaussian distributions. Therefore, it is important to extend the analysis of the kNN density estimators to other types of distributions. 

In this paper, we analyze the $\ell_1$ and uniform convergence rate of the kNN density estimator for a broad range of distributions. Our analysis of the $\ell_1$ error can be easily generalized to $\ell_p$ with $p>1$. To the best of our knowledge, this is the first attempt to analyze the $\ell_1$ and uniform convergence rates of the kNN density estimator in general. Our analysis involves two different cases, depending on whether the support is bounded or not. For both cases, our analysis includes an upper bound of the estimation error of the kNN method, and a minimax lower bound on the performance of all methods. The analysis of both upper and lower bounds is based on some assumptions on the smoothness of the pdf, as well as an additional assumption on the shape of the boundary or the strength of the tail. 

In the first case, the pdf has bounded support and is bounded away from zero. This indicates that the density has boundaries. For example, uniform distribution and truncated Gaussian distribution belong to this case. If the shape of the support set is unknown, the estimation error near the boundary of the support will be relatively large. We show that the $\ell_1$ error converges with the minimax optimal rate, and the error due to the boundary effect will not make the convergence rate of the $\ell_1$ error worse. However, the impact of the boundary effect on the $\ell_\infty$ error is much more serious. To be more precise, the $\ell_\infty$ bound does not converge to zero. This is inevitable since without the knowledge of the support set, it is impossible to design a density estimator that ensures uniform consistency. 
If we have full knowledge of the shape of the support set and the boundary, then we can slightly modify the kNN estimator to correct the estimation bias at the region near the boundary. With the boundary correction, we show that the $\ell_\infty$ error converges to zero and the convergence rate is nearly minimax optimal. We remark that, for the kernel density estimator, there are also some boundary correction methods based on data reflection and transformation \cite{karunamuni2005generalized,hirukawa2010nonparametric}, but the $\ell_1$ or $\ell_\infty$ rates of these methods have not been established. 

In the second case, the pdf is smooth everywhere, and can approach zero arbitrarily close. For example, Gaussian distributions belong to this case. Since the pdf is smooth everywhere, boundary correction is no longer necessary. However, the density estimation is no longer accurate at the tail of the distribution. The reason is that $\hat{f}(\mathbf{x})$ can actually be viewed as an estimate of the average pdf in the neighborhood of $\mathbf{x}$ with the radius equal to the $k$ nearest neighbor distance of $\mathbf{x}$, hence the estimation bias depends on whether the pdf in such neighborhood is sufficiently uniform. If $f(\mathbf{x})$ is very low, then the kNN distance and thus the size of the neighborhood will be large. As a result, the density in the neighborhood of $\mathbf{x}$ is far from uniform, and thus the average pdf in the neighborhood of $\mathbf{x}$ can deviate from $f(\mathbf{x})$ significantly, which will cause a large estimation bias. If the criterion is the $\ell_\infty$ error, we do not need to worry about the bias occurring at the tail of the distribution, since both $\hat{f}(\mathbf{x})$ and $f(\mathbf{x})$ are small. Therefore, we can just use a simple kNN estimator and derive its convergence rate. However, if we use the $\ell_1$ error as the performance criterion, then we need to consider the estimation error over the whole support, instead of only considering its supremum value. As a result, the tail effect is serious and the $\ell_1$ error does not converge to zero. To solve this problem, we design a truncation of the kNN estimator and derive the convergence rate of the $\ell_1$ error for this truncated kNN density estimator. Our analysis shows that under the $\ell_1$ criterion, if the first and second order derivatives of the pdf decay simultaneously with the pdf itself, then the kNN estimator has a better $\ell_1$ convergence rate than the kernel density estimator, although there is still some gap to the minimax lower bound. This result appears to contradict with previous studies such as \cite{mack1983rate}, which claims that the kNN estimator performs worse than the kernel density estimator since it does not handle the tail well. However, the difference is that previous analysis is based on the assumption on the uniform bound of the Hessian, while we assume that the distribution has decaying gradient and Hessian, which holds for many common distributions, such as Gaussian, exponential and Cauchy distributions etc. 

The remainder of this paper is organized as follows. In Section~\ref{sec:estimator}, we provide a simple description of the kNN density estimator. The convergence properties of the kNN density estimator for distributions of the first and the second cases are discussed in Section~\ref{sec:case1} and Section~\ref{sec:case2}, respectively. We then provide some numerical examples in Section~\ref{sec:numerical}. Finally, in Section~\ref{sec:conc}, we offer concluding remarks. 

\section{KNN Density Estimator}\label{sec:estimator}

Consider a distribution with an unknown pdf $f:\mathbb{R}^d\rightarrow \mathbb{R}$. There are $N$ i.i.d samples, $\mathbf{X}_1,\ldots,\mathbf{X}_N$. Our goal is to estimate the pdf $f$ using these samples. For each point $\mathbf{x}\in S$, in which $S$ is the support set of the random variable, denote $\rho(\mathbf{x})$ as the distance between $\mathbf{x}$ and its $k$-th nearest neighbor among $\{\mathbf{X}_1,\ldots,\mathbf{X}_N \}$, in which $k\geq 2$. Then we construct the kNN density estimator as follows:
\begin{eqnarray}
\hat{f}(\mathbf{x})=\frac{k-1}{NV(B(\mathbf{x},\rho(\mathbf{x})))},
\label{eq:fhat}
\end{eqnarray}
in which $B(\mathbf{x},\rho(\mathbf{x}))$ is the ball with center at $\mathbf{x}$ and radius $\rho(\mathbf{x})$, while $V(B(\mathbf{x},\rho(\mathbf{x})))$ denotes the volume of $B(\mathbf{x},\rho(\mathbf{x}))$. 

An intuitive explanation of \eqref{eq:fhat} is that the estimator constructed in \eqref{eq:fhat} is approximately unbiased. Denote $P(B(\mathbf{x},\rho(\mathbf{x})))$ as the probability mass in $B(\mathbf{x},\rho(\mathbf{x}))$, then from order statistics \cite{david1970order}, we know that $P(B(\mathbf{x},\rho(\mathbf{x})))$ follows Beta distribution $\text{Beta}(k,N-k+1)$. As a result, we have
\begin{eqnarray}
\mathbb{E}\left[\frac{1}{P(B(\mathbf{x},\rho(\mathbf{x})))}\right]=\frac{N}{k-1},
\end{eqnarray}  
therefore with approximation $P(B(\mathbf{x},\rho(\mathbf{x})))\approx f(\mathbf{x})V(B(\mathbf{x},\rho(\mathbf{x})))$,
\begin{eqnarray}
\mathbb{E}[\hat{f}(\mathbf{x})]\approx\frac{k-1}{N}\mathbb{E}\left[\frac{f(\mathbf{x})}{P(B(\mathbf{x},\rho(\mathbf{x})))}\right]=f(\mathbf{x}).
\label{eq:approx}
\end{eqnarray}
If the pdf is uniform in $B(\mathbf{x},\rho(\mathbf{x}))$, then $P(B(\mathbf{x},\rho(\mathbf{x})))=f(\mathbf{x})V(B(\mathbf{x},\rho(\mathbf{x})))$. In this case, the first step in \eqref{eq:approx} holds with equal sign, which means that the kNN density estimator \eqref{eq:fhat} is unbiased at $\mathbf{x}$. Note that it is impossible that $P(B(\mathbf{x},\rho(\mathbf{x})))=f(\mathbf{x})V(B(\mathbf{x},\rho(\mathbf{x})))$ holds uniformly for all $\mathbf{x}$ and $\rho(\mathbf{x})$. In particular, the difference between the average pdf in $B(\mathbf{x},\rho(\mathbf{x}))$ and the pdf at its center $\mathbf{x}$ comes from two sources. Firstly, $B(\mathbf{x},\rho(\mathbf{x}))$ may exceed the boundary of the support, thus the average pdf is lower than $f(\mathbf{x})$. Secondly, even if $B(\mathbf{x},\rho(\mathbf{x}))$ is a subset of the support set, the pdf in $B(\mathbf{x},\rho(\mathbf{x}))$ may not be uniform. Both sources are considered in our analysis.

Our analysis includes the bound of the estimation error under both $\ell_1$ and $\ell_\infty$ criteria. The $\ell_1$ error is defined as 
$$\norm{\hat{f}-f}_1=\int_S |\hat{f}(\mathbf{x})-f(\mathbf{x})|d\mathbf{x},$$
and the $\ell_\infty$ error is defined as $$\norm{\hat{f}-f}_\infty =\underset{\mathbf{x}\in S}{\sup} |\hat{f}(\mathbf{x})-f(\mathbf{x})|.$$ 

If $k$ is chosen properly, both the $\ell_1$ and $\ell_{\infty}$ errors of the kNN estimator (or slightly modified kNN estimator, as will be explained in details in the sequel) will go to zero as the number of samples $N$ increases. In this paper, we will analyze the convergence rates at which these errors converge to zero for two different types of distributions: distributions with bounded supports and distributions with unbounded supports.  

\section{Distributions with Bounded Support}\label{sec:case1}
In this section, we analyze the convergence rate of the kNN density estimator for distributions that have bounded supports and the pdfs are bounded away from zero. In particular, we assume that $f(\mathbf{x})>0$ only for $\mathbf{x}\in S$, in which $S\subset \mathbb{R}^d$ is a bounded set.

The analysis is based on the following assumption.

\begin{ass}\label{ass:bounded}
	Assume that the following conditions hold:
	
	(a) $f$ is upper bounded, and is also bounded away from zero, i.e. there exist two constants $m$ and $M$, such that $m\leq f(\mathbf{x})\leq M$ for all $\mathbf{x}\in S$;
	
	(b) $f$ is $L$-Lipschitz, i.e. for all $\mathbf{x},\mathbf{x}'\in S$,
	\begin{eqnarray}
	|f(\mathbf{x})-f(\mathbf{x}')|\leq L\norm{\mathbf{x}-\mathbf{x}'};
	\end{eqnarray}
	
	(c) The surface area (or Hausdorff measure) of $S$ is no more than $C_S$.
\end{ass}
In Assumption \ref{ass:bounded}, we assume in (a) that the pdf is both bounded above and is also bounded away from zero. This assumption is necessary, since the convergence rate will be slower if the pdf can approach zero arbitrarily close. The case where pdf can approach zero will be analyzed in Section~\ref{sec:case2}. (b) bounds the gradient of the pdf, which can decide the accuracy of the approximation in \eqref{eq:approx}. It would be tempting to consider some more general smoothness classes for $f$. For example, some distributions may be second order continuous, which means that both $\norm{\nabla f}$ and $\norm{\nabla^2 f}$ is bounded above. However, with the standard kNN algorithm, the $\ell_\infty$ convergence rate will not be further improved comparing with only assuming the bounded gradient. The reason is that we are bounding the supremum of the estimation error, which usually happens at the region near the boundary of the support of the distribution. If we use the $\ell_1$ criterion instead, then it is possible that the convergence rate can be improved for distributions with higher smoothness level. However, for simplicity, we only assume that $f$ is Lipschitz here. Moreover, in (c), we assume the boundedness of the surface area in (c). This assumption is important because it restricts the volume of the region near the boundary, and is thus crucial to bound the estimation error due to the boundary effect.
\subsection{$\ell_1$ bound}
To begin with, we show the convergence rate of the $\ell_1$ error for distributions with bounded supports. The result is shown in Theorem \ref{thm:compact}. Throughout the paper, notation $a\lesssim b$ means that there exists a constant $C$ such that $a\leq Cb$. 
$a\gtrsim b$ is defined in a similar manner. 

\begin{thm}\label{thm:compact}	
	Under Assumption \ref{ass:bounded}, the kNN density estimator \eqref{eq:fhat} satisfies the following bound:
	\begin{eqnarray}
	\mathbb{E}\left[\norm{\hat{f}-f}_1\right]\lesssim \left(\frac{k}{N}\right)^\frac{1}{d}+k^{-\frac{1}{2}}.
	\label{eq:compact1}
	\end{eqnarray}	
	Moreover, define $\Sigma_A(S)$ as the set of all distributions with support set $S$ that satisfy Assumption \ref{ass:bounded}. If $L,M$ are sufficiently large and $m$ is sufficiently small, then
	\begin{eqnarray}
	\underset{\hat{f}}{\inf}\underset{f\in \Sigma_A(S)}{\sup} \mathbb{E}\left[\norm{\hat{f}-f}_1\right]&\gtrsim& N^{-\frac{1}{d+2} }.
	\label{eq:cmplb1}	
	\end{eqnarray}	
\end{thm}
\begin{proof}
Please see Appendix~\ref{sec:bounded-l1} for details.
\end{proof}

In Theorem \ref{thm:compact}, the upper bound \eqref{eq:compact1} can be proved by bounding the bias due to the two sources mentioned above, including the boundary bias and the bias caused by the local nonuniformity of the pdf. After that, the random estimation error $\hat{f}-\mathbb{E}[\hat{f}]$ can be bounded using techniques from order statistics \cite{david1970order}. The detailed proof is shown in Appendix \ref{sec:bounded-l1}. The lower bound \eqref{eq:cmplb1} can be shown simply by standard minimax analysis techniques in \cite{tsybakov2009introduction}. 

Comparing the upper bound \eqref{eq:compact1} and the minimax lower bound in \eqref{eq:cmplb1}, it can be observed that if $k\sim N^{2/(d+2)}$, then the convergence rate of the estimation error of the kNN density estimator under $\ell_1$ is minimax optimal. Note that in Theorem \ref{thm:compact}, we do not assume the knowledge of the support set to achieve the upper bound \eqref{eq:compact1}. However, for the minimax lower bound \eqref{eq:cmplb1}, the support set $S$ is assumed to be known. The upper bound \eqref{eq:compact1} and the lower bound \eqref{eq:cmplb1} still match, even if the lower bound is derived under a more restrictive condition than the upper bound. This result indicates that for the $\ell_1$ bound, the boundary bias does not make the convergence rate of the kNN density estimation worse, even if the support is unknown and boundary correction methods have not been implemented. An intuitive explanation is that with the increase of sample size $N$, the kNN distances $\rho(\mathbf{x})$ becomes smaller, hence the probability that $B(\mathbf{x},\rho(\mathbf{x}))$ exceeds the boundary of the support becomes lower, and correspondingly, the convergence rate of the bias due to the boundary effect is the same as that due to the local nonuniformity of the density. As a result, the $\ell_1$ error performance of the kNN density estimator is not seriously affected by the boundary effect.

\subsection{$\ell_\infty$ bound}
Unlike the $\ell_1$ bound, the $\ell_\infty$ bound of the original kNN density estimator does not converge to zero. The reason is that if $\mathbf{x}$ is near the boundary, on which $f(\mathbf{x})$ changes sharply, the approximation in \eqref{eq:approx} does not hold and the bias can be large. Such boundary effect does not affect the convergence rate of the $\ell_1$ bound, since the $\ell_1$ bound is the integration of estimation error over the whole support, and the region such that the boundary effect occurs shrinks with the sample size $N$. However, if we use the $\ell_\infty$ criterion, which only evaluates the maximum estimation error over the whole support, then the boundary bias becomes crucial. To correct the bias, we design the following estimator:
\begin{eqnarray}
\hat{f}_{BC}(\mathbf{x})=\frac{k-1}{NV_S(B(\mathbf{x},\rho(\mathbf{x})))},
\label{eq:fbc}
\end{eqnarray}
in which $\hat{f}_{BC}$ means the boundary corrected estimator, and $V_S(B(\mathbf{x},\rho(\mathbf{x})))=V(B(\mathbf{x},\rho(\mathbf{x}))\cap S)$. 

\begin{thm}\label{thm:compact-unif}
	Under Assumption \ref{ass:bounded}, if the support $S$ is known, using the boundary corrected estimator \eqref{eq:fbc}, with probability at least $1-\epsilon$, the $\ell_\infty$ bound satisfies
	\begin{eqnarray}
	\norm{\hat{f}_{BC}-f}_\infty \lesssim \left(\frac{k}{N}\right)^\frac{1}{d}+k^{-\frac{1}{2}}\sqrt{\ln \frac{N}{\epsilon}}.
	\label{eq:compact-unif}
	\end{eqnarray}	
	
	Moreover, define $\Sigma_A$ as the set of all distributions with arbitrary support sets that satisfy Assumption \ref{ass:bounded}, and define $\Sigma_A(S)$ the same as in Theorem \ref{thm:compact}. The difference between $\Sigma_A$ and $\Sigma_A(S)$ is that the latter one has a fixed support $S$. If $L,M,H$ are sufficiently large and $m$ is sufficiently small, then
	\begin{eqnarray}
	\underset{\hat{f}}{\inf}\underset{f\in \Sigma_A}{\sup} \mathbb{E}\left[\norm{\hat{f}-f}_\infty\right]&\gtrsim& 1;
	\label{eq:cmplb2}\\
	\underset{\hat{f}}{\inf}\underset{f\in \Sigma_A(S)}{\sup} \mathbb{E}\left[\norm{\hat{f}-f}_\infty\right]&\gtrsim& N^{-\frac{1}{d+2}}.
	\label{eq:cmplb3}	
	\end{eqnarray}	
\end{thm}
\begin{proof}
	Please see Appendix~\ref{sec:bounded-unif} for details.
\end{proof}

In Theorem \ref{thm:compact-unif}, \eqref{eq:compact-unif} provides an upper bound of the boundary corrected kNN density estimator \eqref{eq:fbc}. For the proof of \eqref{eq:compact-unif}, we use the following steps. Firstly, we construct some grid points in the support. Then we find the uniform bound of estimation error among these grid points. Finally, we generalize the uniform bound among finite number of grid points to the whole space. We let the number of grid points increase with the number of samples, so that the extra estimation error due to the generalization is not large. The detailed proof is shown in Appendix \ref{sec:bounded-unif}. Moreover, \eqref{eq:cmplb2} and \eqref{eq:cmplb3} provide the minimax lower bound of the $\ell_\infty$ error with unknown and known support set, respectively. \eqref{eq:cmplb2} can be shown by simply using Le Cam's lemma \cite{tsybakov2009introduction}, while \eqref{eq:cmplb3} can be proved easily by standard minimax analysis \cite{tsybakov2009introduction}. We provide a simple proof of \eqref{eq:cmplb2} at the end of Appendix \ref{sec:bounded-unif}, and omit the detailed proof of \eqref{eq:cmplb3} for simplicity. According to \eqref{eq:cmplb2}, if the support set $S$ is unknown, then it is not possible to construct an estimator with the $\ell_\infty$ error converging to zero. If the support set is known, then the minimax lower bound becomes \eqref{eq:cmplb3}. Comparing with \eqref{eq:compact-unif}, it can be observed that the kNN density estimator with boundary correction \eqref{eq:fbc} is minimax rate optimal. 

We remark that the convergence rate derived in Theorem \ref{thm:compact-unif} appears to be slower than the result in \cite{mack1983rate}. In particular, \cite{mack1983rate} assumes that the second order derivative of $f$ exists and is bounded, then its eq.(k2) and eq.(7) show that it is possible to select an appropriate $k$, so that the convergence rate can be made faster. However, the analysis in \cite{mack1983rate} did not take the boundary effect into consideration. In fact, using similar techniques in Theorem \ref{thm:compact-unif}, we can show that the uniform convergence rate of the kNN density estimator for distributions with bounded support does not improve even if the second order derivative of $f$ exists and is bounded, since the boundary bias is actually dominant in this case.

\section{KNN Density Estimator for Distributions with Unbounded Support}\label{sec:case2}
In this section, we investigate the $\ell_1$ and uniform convergence of the kNN density estimator for distributions that are smooth everywhere and have unbounded support. For these distributions, the pdf can approach zero arbitrarily close in its tail, at which kNN distances are usually large and the approximation in \eqref{eq:approx} no longer holds, i.e. the average pdf in the neighborhood of $\mathbf{x}$ can be far away from $f(\mathbf{x})$ at the tail of the distribution. As a result, the density estimation at the tails is hard. Unlike the case with bounded support, the assumptions for deriving $\ell_1$ and $\ell_\infty$ bounds are slightly different, hence we state the assumptions separately in Theorem \ref{thm:unbounded} and Theorem \ref{thm:unbounded-unif}.

\subsection{$\ell_1$ bound}
Now we analyze the convergence rate of the $\ell_1$ error. To begin with, we show that the $\ell_1$ error of the original kNN estimator defined in \eqref{eq:fhat} is actually infinite. Recall that $\mathbf{X}_i$, $i=1,\ldots, N$ are the samples for density estimation. Define $R$ as their maximum distance to $\mathbf{x}=0$, i.e.
\begin{eqnarray}
R=\max_{i=1,\ldots, N}\{\norm{\mathbf{X}_i} \}.
\end{eqnarray}
Then for all $\mathbf{x}$ such that $\norm{\mathbf{x}}>R$, we have $\rho(\mathbf{x})<\norm{\mathbf{x}}+R$, since the distance of all the samples can not be more than $\norm{\mathbf{x}}+R$ away from $\mathbf{x}$. Hence
\begin{eqnarray}
\int \hat{f}(\mathbf{x})d\mathbf{x}\geq\int_{\norm{\mathbf{x}}>R} \hat{f}(\mathbf{x})d\mathbf{x}\geq \frac{k-1}{Nv_d}\int_{\norm{\mathbf{x}}>R} \frac{1}{(\norm{\mathbf{x}}+R)^d}d\mathbf{x}=\infty.
\end{eqnarray}
The above result shows that the $\ell_1$ error of the original kNN density estimator is always infinite, and is thus not suitable for distributions with tails. In fact, the estimated pdf does not decay sufficiently fast with the pdf itself. As a result, the estimation error at the tail distribution is serious. To improve the performance of the kNN density estimator, we design a truncated estimator as following:
\begin{eqnarray}
\hat{f}(\mathbf{x})=\left\{
\begin{array}{ccc}
\frac{k-1}{NV(B(\mathbf{x},\rho(\mathbf{x})))} & \text{if} & \rho(\mathbf{x})\leq a\\
\frac{n(\mathbf{x},a)}{NV(B(\mathbf{x},a))} &\text{if} &\rho(\mathbf{x})>a,
\end{array}
\right.
\label{eq:ft}
\end{eqnarray} 

in which
\begin{eqnarray}
n(\mathbf{x}, a) = \sum_{i=1}^N \mathbf{1}(\mathbf{X}_i\in B(\mathbf{x},a))
\end{eqnarray}
is the number of samples in $B(\mathbf{x},a)$. 

This new estimator is designed such that if the kNN distance $\rho(\mathbf{x})$ is larger than a threshold $a$, then the estimated value will be replaced by counting the number of samples that falls in a ball with radius $a$. Intuitively, this design ensures that the distances from $\mathbf{x}$ to all samples involved in the density calculation are no more than $a$, therefore, it avoids the kNN distance from being too large, and can thus reduce the estimation bias. The choice of $a$ will be provided in the sequel. We now bound the convergence rate of the $\ell_1$ error of kNN density estimation. The results are summarized in Theorem \ref{thm:unbounded}.

\begin{thm}\label{thm:unbounded}
	Assume that there exist four constants $C_b$, $C_c$, $C_d$ and $\beta \in (0,1]$, such that
	
	(a) $f(\mathbf{x})\leq 1$;
	
	(b) The gradient of pdf satisfies
	\begin{eqnarray}
	\norm{\nabla f(\mathbf{x})}\leq C_b f(\mathbf{x})\left(1+\ln \frac{1}{f(\mathbf{x})}\right);
	\end{eqnarray}
	
	(c) The Hessian of pdf satisfies
	\begin{eqnarray}
	\norm{\nabla^2 f(\mathbf{x})}_{op}\leq C_c f(\mathbf{x})\left(1+\ln \frac{1}{f(\mathbf{x})}\right),
	\end{eqnarray}
	in which $	\norm{\cdot}_{op}$ denotes the operator norm;
	
	(d) For any $t>0$,
	\begin{eqnarray}
	\text{P}(f(\mathbf{X})<t)\leq C_d t^\beta.
	\end{eqnarray}
	
	If we set $a\sim N^{-\frac{\beta'^2}{d\beta'^2+1}}$ and $k\sim N^{\frac{2\beta'}{d\beta'^2+1}}$, in which $\beta'=\min\{\beta,1/2 \}$, then
	\begin{eqnarray}
	\mathbb{E}\left[\norm{\hat{f}-f}_1\right]\lesssim N^{-\min\left\{\frac{\beta}{d\beta^2 +1}, \frac{2}{d+4} \right\}}\ln^2 N.
	\label{eq:unbounded}
	\end{eqnarray}
\end{thm}

\begin{proof}
	Please refer to Appendix \ref{sec:unbounded-l1} for details.
\end{proof}

In assumption (a), we set the maximum $f(\mathbf{x})$ to be $1$ just for convenience. (b) and (c) assume that the first and second order derivatives decay simultaneously with $f(\mathbf{x})$. This holds for many common distributions. For example, for Gaussian distributions $f\sim \exp(-c\norm{\mathbf{x}}^2)$, both (b) and (c) are satisfied. These two assumptions ensure that the bias of the kNN density estimator is not too large. (d) restricts the tail strength of the distribution. A smaller $\beta$ indicates that the tail is stronger. We assume that $\beta\leq 1$, since if $\beta>1$, it can be proved that the support set is bounded, while here we hope to analyze distributions with unbounded support. In fact, from assumption (b) and (c), it can be shown that $f(\mathbf{x})>0$ everywhere, and thus the support must be unbounded. Now we provide some examples of distributions satisfying assumption (d). For one or two dimensional Gaussian distributions, assumption (d) is satisfied for $\beta=1$. For Gaussian distributions with higher dimensions, assumption (d) is satisfied for $\beta$ arbitrarily close to $1$. For Cauchy distributions, assumption (d) is satisfied with $\beta=1/2$. For $t_n$ distributions, assumption (d) is satisfied with $\beta=n/(n+1)$. Moreover, if a distribution has finite moments up to infinite order, i.e. $\mathbb{E}[\norm{\mathbf{X}}^\alpha]<\infty$ for all $\alpha>0$, then assumption (d) holds for all $\beta<1$.

For the proof of Theorem \ref{thm:unbounded}, we bound $\mathbb{E}[|\hat{f}(\mathbf{x})-f(\mathbf{x})|]$ separately depending on the value of $f(\mathbf{x})$. If $f(\mathbf{x})$ is sufficiently large, then with high probability, $\rho(\mathbf{x})\leq a$. As a result, the kNN estimator is not truncated. On the other hand, if the pdf is low, then the probability that kNN estimator is truncated is not negligible. We bound $\mathbb{E}[|\hat{f}(\mathbf{x})-f(\mathbf{x})|]$ using different methods. The detailed proof is shown in Appendix \ref{sec:unbounded-l1}.

Now we show the minimax lower bound of the $\ell_1$ error.
\begin{thm}\label{thm:unbounded-lb}
	Define $\Sigma_B$ as the set of all functions that satisfy assumption (a)-(d) in Theorem \ref{thm:unbounded}, if $C_b$, $C_c$, $C_d$ are sufficiently large, then
	\begin{eqnarray}
	\underset{\hat{f}}{\inf}\underset{f\in \Sigma_B}{\sup} \mathbb{E}\left[\norm{\hat{f}-f}_1\right]\gtrsim N^{-\min\left\{\beta,\frac{2}{d+4} \right\}}.
	\end{eqnarray}
\end{thm}
\begin{proof}
	Please refer to Appendix \ref{sec:minimax} for details.
\end{proof}


Comparing Theorem \ref{thm:unbounded} with Theorem \ref{thm:unbounded-lb}, we observe that if $\beta>\min\{1/2,2/d\}$, then we can let $k\sim N^{4/(4+d)}$ and the convergence rate of the $\ell_1$ error matches the minimax lower bound up to a logarithmic factor. With a smaller $\beta$, there exists some gap between the upper bound and the lower bound, indicating that it is still possible to improve the convergence rate.

Despite that the truncated kNN density estimator \eqref{eq:ft} has some gap to the minimax optimal rate for small $\beta$, we would like to note that the performance of the truncated kNN density estimator \eqref{eq:ft} is better than the kernel density estimator for distributions with heavy tails. To be more precise, we have the following Proposition.
\begin{prop}\label{prop}
	For a kernel density estimator
	\begin{eqnarray}
	\hat{f}(\mathbf{x})=\frac{1}{Nh^d}\sum_{i=1}^N K\left(\frac{\mathbf{X}_i-\mathbf{x}}{h}\right),
	\end{eqnarray}
	in which $K(\cdot)$ is supported on $B(\mathbf{0},1)$, $\int K(\mathbf{u})d\mathbf{u}=1$ and $K(\mathbf{u})\leq K_m$ for some constant $K_m$. If $C_b$, $C_c$, $C_d$ are sufficiently large, then
	\begin{eqnarray}
	\underset{h}{\inf}\underset{f\in \Sigma_B}{\sup}\mathbb{E}\left[\norm{\hat{f}-f}_1\right]\gtrsim N^{-\min\left\{\frac{2\beta}{2+d\beta},\frac{2}{d+4} \right\}}.
	\label{eq:kde}
	\end{eqnarray}
\end{prop}
\begin{proof}
Please refer to Appendix \ref{sec:kde} for details.
\end{proof}

In \eqref{eq:kde}, we take the supremum over all distributions satisfying assumptions (a)-(d) in Theorem \ref{thm:unbounded}, and take the infimum over all possible $h$. The rate in the right hand side of \eqref{eq:kde} indicates the theoretical limit such that the kernel density estimator can not perform better than this limit for any bandwidth $h$. This can be proved by analyzing the bias and the random error separately. Note that $\mathbb{E}[\hat{f}(\mathbf{x})]=f\star K_h$, in which $\star$ denotes convolution and $K_h(\cdot)=K(\cdot/h)/h^d$. The convolution will induce roughly $h^2$ bias. We also provide a lower bound of the random error. The detailed proof is shown in Appendix \ref{sec:kde}. 

Comparing \eqref{eq:kde} with \eqref{eq:unbounded}, it can be observed that if $\beta\geq 1/2$, then the truncated kNN density estimator and the kernel density estimator have the same convergence rate and are both minimax optimal, except a logarithm factor. For distributions with heavy tails such that $\beta<1/2$, the truncated kNN density estimator performs better than the kernel density estimator. For high dimensional random variables, such difference is more obvious, since from \eqref{eq:unbounded} and \eqref{eq:kde}, if $2/d< \beta< 1/2$, then the truncated kNN estimator is minimax rate optimal, while the kernel density estimator is not optimal. In some previous literatures such as \cite{mack1983rate}, it was believed that the kNN performs worse than kernel density estimator for distributions with heavy tails. However, the previous analysis is based on the uniform bound of Hessian, while in our assumption (b) and (c), the gradient and Hessian also decay with the pdf. As a result, the comparison between these two estimators are reversed due to the difference of assumptions. We provide an intuitive explanation of the reason why the kNN estimator has a better convergence rate than the kernel density estimator as following. In the tail of the distribution, the kNN distances are large, while for the kernel density estimation, the kernel size is constant all over $\mathbb{R}^d$. As a result, comparing with the kernel density estimator, the kNN method has a larger bias but smaller variance at the tail of the distribution. If the pdf only has bounded Hessian without decaying, than the larger bias of the kNN method is more obvious. However, under our assumption, the Hessian decays with roughly the same rate as the pdf $f$, hence the bias will not increase much, and thus the kNN method achieves a better tradeoff between bias and variance than the kernel density estimator.

\subsection{$\ell_\infty$ bound}

We now analyze the uniform convergence rate of the kNN density estimator. For the uniform convergence rate, we only care about the maximum estimation error. As a result, truncation is not necessary, hence we just use the simple kNN density estimator \eqref{eq:fhat}. The result is shown in Theorem \ref{thm:unbounded-unif}.
\begin{thm}\label{thm:unbounded-unif}
	Suppose $f$ satisfies assumption (a), (b) and (c) in Theorem \ref{thm:unbounded}, and the following additional assumption
	\begin{eqnarray}
	\mathcal{N}\left(\{\mathbf{x}|f(\mathbf{x})>m\},r \right)\leq \frac{\mathcal{N}_0}{m^\gamma r^d},
	\label{eq:additional}
	\end{eqnarray}
	for some $\gamma>0$ and all $m>0$, in which $\mathcal{N}$ denotes the covering number. Then with probability at least $1-\epsilon$,
	\begin{eqnarray}
	\underset{\mathbf{x}}{\sup}|\hat{f}(\mathbf{x})-f(\mathbf{x})|\lesssim \left\{
	\begin{array}{ccc}
	\left(\frac{k}{N}\right)^\frac{2}{d}+k^{-\frac{1}{2}}\sqrt{\ln \frac{N}{\epsilon}} &\text{if} & d>2\\
	\frac{k}{N}\ln^d N+k^{-\frac{1}{2}}\sqrt{\ln \frac{N}{\epsilon}} &\text{if} & d=1,2.
	\end{array}
	\right.
	\end{eqnarray}
\end{thm}
\begin{proof}
	Please refer to Appendix \ref{sec:unbounded-unif} for details.
\end{proof}
In Theorem \ref{thm:unbounded-unif}, we do not have the assumption (d) in Theorem \ref{thm:unbounded}. Actually, the tail strength does not affect the uniform convergence rate, since the $\ell_\infty$ bound only cares about the supremum error. However, we impose another assumption on the regularity of $\{\mathbf{x}|f(\mathbf{x})\geq m\}$. This additional assumption is actually very weak and is satisfied by almost all pdfs. 

The proof of Theorem \ref{thm:unbounded-unif} can be divided into two parts. Firstly, in the region with high pdf, the uniform convergence rate can be bounded using similar techniques as is used in the proof of Theorem \ref{thm:compact-unif}, which involves constructing some grid points, finding the uniform bound in the grid points, and then generalizing to the overall uniform bound over the whole region. However, since the support is unbounded, such technique can not be simply generalized to the whole space $\mathbb{R}^d$, especially to the region with low density, since the number of grid points will be infinite, and thus the related union bound does not work. Hence, we provide the uniform bound of kNN estimator by finding the lower bound of the kNN distances. 

The corresponding minimax lower bound is shown in Theorem \ref{thm:unbounded-unif-lb}.
\begin{thm}\label{thm:unbounded-unif-lb}
	Define $\Sigma_C$ as the set of all functions that satisfy assumption (a)-(c) in Theorem \ref{thm:unbounded} and the additional assumption \eqref{eq:additional}, then
	\begin{eqnarray}
	\underset{\hat{f}}{\inf}\underset{f\in \Sigma_C}{\sup} \mathbb{E}\left[\norm{\hat{f}-f}_\infty\right]\gtrsim N^{-\frac{2}{d+4} }.
	\end{eqnarray}
\end{thm}
\begin{proof}
	Please see Appendix \ref{sec:unbounded-unif-lb} for the detailed proof.
\end{proof}

We observe that if $d\geq 2$, with a proper selection of $k$, i.e. $k\sim N^{4/(d+4)}$, the upper bound of the kNN density estimator \eqref{eq:fhat} nearly matches the minimax lower bound. If $d=1$, then the upper bound does not match the minimax lower bound. 

\section{Numerical Examples}\label{sec:numerical}

In this section, we provide several numerical experiments to illustrate  the theoretical results derived in this paper. Our simulation has three parts. 

In the first part, we show the convergence plots of the $\ell_1$ and $\ell_\infty$ estimation errors of the boundary corrected kNN density estimator \eqref{eq:fbc} for uniform distributions, which is a typical example of distributions with bounded support. In the simulation, $k$ is selected to minimize the $\ell_1$ and $\ell_\infty$ error. The optimal growth rate of $k$ determined by Theorem \ref{thm:compact} and \ref{thm:compact-unif} are the same, i.e. $k\sim N^\frac{2}{d+2}$ can optimize both $\ell_1$ and $\ell_\infty$ rate. Therefore, we use this rate in the simulation. This part is shown in Figure~\ref{fig:compare} (a) and (b). 

In the second part, we show the convergence plots for Gaussian distributions, which is an example of distributions with unbounded support, as is shown in Figure~\ref{fig:compare}  (c) and (d). Note that according to Theorem \ref{thm:unbounded} and Theorem \ref{thm:unbounded-unif}, the optimal growth rate of $k$ that optimizes the $\ell_1$ and $\ell_\infty$ errors is different if $d=1$. For simplicity, we select $k$ that only optimizes $\ell_\infty$. For the first and the second part, for each $k$ and each sample size $N$, our simulation involves the following steps.

(1) Generate $N$ i.i.d samples according to a distribution, such as the standard Gaussian distribution;

(2) Find a region on which the probability mass of the distribution is sufficiently close to $1$. For example, for one dimensional standard Gaussian distribution, this region can be $[-5,5]$. Then divide the region into grids of size $0.01$;

(3) For each grid point, estimate its pdf value using the kNN density estimation method, and find its difference with the true value. Calculate the average and the maximum of such difference over all grids, in which the former one can be used as an estimate of the $\ell_1$ error by multiplying an appropriate factor, while the latter one can be used as an estimate of the $\ell_\infty$ error;

(4) Repeat (1)-(3) for $T=5000$ times, and find the average $\ell_1$ and $\ell_\infty$ error.

In the third part, we compare the $\ell_1$ error of the kNN density estimator and the kernel density estimator for two heavy tailed distributions. One is the Cauchy distribution, $f_1(x)=1/(\pi (1+x^2 ))$, and the other one is $f_2(x)=(|x|+1)^{-2/3}/4$. In our experiment, if the dimension is higher than $1$, then the high dimensional distribution is just the simple joint of i.i.d one dimensional distributions. The parameters for both methods are tuned optimally in the simulation. In Fig.\ref{fig:compare} (e) and (f), we plot the ratio between the $\ell_1$ error of the truncated kNN \eqref{eq:ft} and the kernel density estimators. If the ratio is lower than $1$, then the kNN method performs better than the kernel density estimator, and vice versa.
\begin{figure}[h!]
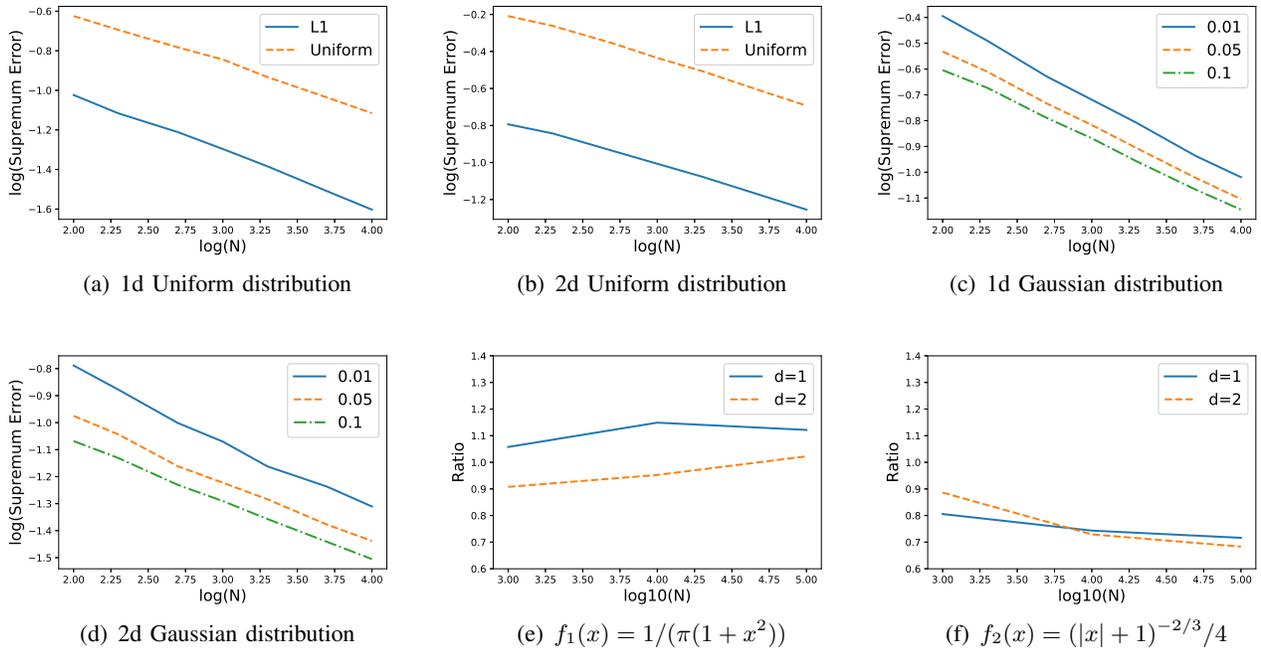

	\begin{center}
		\subfigure[1d Uniform distribution]{\includegraphics[width=0.32\linewidth]{../code/Uniform.eps}}	
		\subfigure[2d Uniform distribution]{\includegraphics[width=0.32\linewidth]{../code/Uniform2.eps}}				
		\subfigure[1d Gaussian distribution]{\includegraphics[width=0.32\linewidth]{../code/Gaussian1.eps}}	
		\subfigure[2d Gaussian distribution]{\includegraphics[width=0.32\linewidth]{../code/Gaussian2.eps}}				
		\subfigure[$f_1(x)=1/(\pi (1+x^2 ))$]{\includegraphics[width=0.32\linewidth]{../code/Cauchy_compare.eps}}	
		\subfigure[$f_2(x)=(|x|+1)^{-2/3}/4$]{\includegraphics[width=0.32\linewidth]{../code/newtype_compare.eps}}				
		\caption{Numerical simulation results of kNN density estimation. (a) and (b) show the convergence plot of the $\ell_1$ and $\ell_\infty$ estimation errors with respect to $N$ for one and two dimensional uniform distributions. (c) and (d) correspond to one and two dimensional Gaussian distributions. In this case, $k\sim N^{2/3}$. (e) and (f) compare the truncated kNN method with the kernel density estimator for two types of heavy tailed distributions. In (e), $f(x)=1/(\pi(1+x^2))$. In (f), $f(x)=(|x|+1)^{-2/3}/4$. The vertical axis is the ratio between the $\ell_1$ error of the kNN method and the $\ell_1$ error of the kernel method. }\label{fig:compare}
	\end{center}
\end{figure}

We further list the empirical and theoretical convergence rates in Table \ref{tab:convergence}. In Table \ref{tab:convergence}, the empirical convergence rates are the negative slopes of the curves in Fig. \ref{fig:compare}(a)-(d), and the theoretical convergence rates are the results in Theorem \ref{thm:compact}, \ref{thm:compact-unif}, \ref{thm:unbounded} and \ref{thm:unbounded-unif}. For simplicity, we only show the exponents in Table \ref{tab:convergence}, and ignore the logarithm factor. To be more precise, we fill $\delta$ in the table if the convergence rate is $\mathcal{\tilde{O}}(N^{-\delta})$.

\begin{table}[h]
	\centering
	\begin{tabular}{c|c|c|c|c}
		\hline
		\multirow{2}{*}{Case}& \multicolumn{2}{c}{$L_1$ error} & \multicolumn{2}{|c}{$L_\infty$ error}\\
		\cline{2-5}
		&Empirical &Theoretical & Empirical & Theoretical\\
		\hline
		Uniform distribution with $d=1$ &0.33&0.33&0.30&0.33\\
		Uniform distribution with $d=2$ & 0.25 & 0.25 & 0.26 & 0.25\\
		Gaussian distribution with $d=1$ & 0.39 & 0.40 & 0.30 & 0.33\\
		Gaussian distribution with $d=2$ & 0.31 & 0.33 & 0.30 & 0.33\\
		\hline
	\end{tabular}
\caption{Empirical and theoretical convergence rates of density estimation}\label{tab:convergence}
\end{table}

The results in Figure~\ref{fig:compare} (a)-(d) and Table~\ref{tab:convergence} show that the empirical convergence rates of the kNN density estimator \eqref{eq:fhat}, the boundary corrected one \eqref{eq:fbc} or the truncated one \eqref{eq:ft} agree with the theoretical analysis in general. From Figure~\ref{fig:compare} (e), it can be observed that for Cauchy distributions, the kNN method appears to perform slightly worse than the kernel density estimator, since the ratio is slightly above $1$. The Cauchy distribution satisfies assumption (d) in Theorem \ref{thm:unbounded} with $\beta=1/2$. According to Theorem \ref{thm:unbounded} and Proposition \ref{prop}, the convergence rates of these two methods are nearly the same. Hence, it is natural to observe some practical differences between the performance of these two estimators. If the tail is heavier, then the performance of the kNN method becomes obviously better than the kernel density estimator. The distribution in Figure~\ref{fig:compare} (f) satisfies assumption (d) in Theorem \ref{thm:unbounded} with $\beta=1/3$. Our theoretical analysis in Theorem \ref{thm:unbounded} and Proposition \ref{prop} indicates that the convergence rate of the truncated kNN estimator is faster than the kernel density estimator under this $\beta$. This can be observed in the curves in Figure~\ref{fig:compare} (f), in which the ratios are all below $1$ and are decaying with the increase of sample size $N$.

\section{Conclusion}\label{sec:conc}

In this paper, we have analyzed the convergence property of the estimation errors of the kNN density estimator under $\ell_1$ and $\ell_\infty$ criteria. The analysis is conducted for two types of distributions, including those with bounded support and the densities bounded away from zero, and those with unbounded support. We have shown the following results:

Firstly, for distributions with bounded support, the kNN density estimator is optimal under the $\ell_1$ criterion, even if the support set is unknown. If we use $\ell_\infty$ as the criterion, then the knowledge of the support is necessary, and a proper boundary correction technique is needed. Without the precise knowledge of the boundary, no estimator is uniformly consistent. After boundary correction, the kNN density estimator is minimax optimal under the $\ell_\infty$ criterion.

Secondly, for distributions with unbounded support, the $\ell_\infty$ bound is nearly minimax optimal. However, under the $\ell_1$ criterion, the original kNN estimator does not work, since the estimated pdf does not decay sufficiently fast with the real pdf. Therefore, a proper truncation is needed. We have derived the convergence rate of such truncated kNN density estimator, as well as the corresponding minimax lower bound. The comparison of these two bounds shows that the kNN density estimator has some gap to the minimax lower bound. However, we remark that the $\ell_1$ convergence rate is better than that of the kernel density estimator. This appears to conflict with previous works, but the previous works only assume the uniform bound of Hessian. If the gradient and Hessian of the pdf do not decay, then the bias at the tail is indeed large. However, for many common distributions, the Hessian decays simultaneously with the pdf. We have compared the convergence rates of these two methods for distributions with decaying Hessian, and have shown that the kNN density estimator with a proper truncation actually performs better.

\appendices

\section{Proof of Theorem~\ref{thm:compact}}\label{sec:bounded-l1}

Recall that
\begin{eqnarray}
\hat{f}(\mathbf{x})=\frac{k-1}{NV(B(\mathbf{x},\rho(\mathbf{x})))}.
\end{eqnarray}
We decompose the estimation error as
\begin{eqnarray}
\hat{f}(\mathbf{x})-f(\mathbf{x})&=&\left[\frac{k-1}{NV(B(\mathbf{x},\rho(\mathbf{x})))}-\frac{k-1}{NP(B(\mathbf{x},\rho(\mathbf{x})))}f(\mathbf{x})\right]+\left[\frac{k-1}{NP(B(\mathbf{x},\rho(\mathbf{x})))}-1\right]f(\mathbf{x})\nonumber\\
&:=& I_1+I_2.
\end{eqnarray}
Therefore
\begin{eqnarray}
\mathbb{E}[|\hat{f}(\mathbf{x})-f(\mathbf{x})|]\leq \mathbb{E}[|I_1|]+\mathbb{E}[|I_2|].
\end{eqnarray}
\textbf{Bound of $I_1$.}
\begin{eqnarray}
\mathbb{E}[|I_1|]=\mathbb{E}\left[\left|\frac{k-1}{NV(B(\mathbf{x},\rho(\mathbf{x})))}-\frac{k-1}{NP(B(\mathbf{x},\rho(\mathbf{x})))}f(\mathbf{x})\right|\right].
\end{eqnarray}
Denote $\Delta(\mathbf{x})$ as the distance from $\mathbf{x}$ to the boundary of $S$, i.e. for all $\mathbf{x}\in S$,
\begin{eqnarray}
\Delta(\mathbf{x}) = \inf\{\norm{\mathbf{x}-\mathbf{u}}|\mathbf{u}\in \partial S \},
\end{eqnarray}
in which $\partial S$ is the boundary of $S$. If $\rho(\mathbf{x})\leq \Delta(\mathbf{x})$, then $B(\mathbf{x},\rho(\mathbf{x}))\subset S$. Since $f$ is Lipschitz,
\begin{eqnarray}
|P(B(\mathbf{x},\rho(\mathbf{x})))-f(\mathbf{x})V(B(\mathbf{x},\rho(\mathbf{x})))|\leq L\rho(\mathbf{x})V(B(\mathbf{x},\rho(\mathbf{x}))).
\end{eqnarray} 
Hence for sufficiently large $k$,
\begin{eqnarray}
&&\mathbb{E}\left[\left|\frac{k-1}{NV(B(\mathbf{x},\rho(\mathbf{x})))}-\frac{k-1}{NP(B(\mathbf{x},\rho(\mathbf{x})))}f(\mathbf{x})\right|\mathbf{1}(\rho(\mathbf{x})\leq \Delta(\mathbf{x}))\right]\nonumber\\
&= & \frac{k-1}{N}\mathbb{E}\left[\frac{1}{P(B(\mathbf{x},\rho(\mathbf{x})))}\left|\frac{P(B(\mathbf{x},\rho(\mathbf{x})))-f(\mathbf{x})V(B(\mathbf{x},\rho(\mathbf{x})))}{V(B(\mathbf{x},\rho(\mathbf{x})))}\right|\mathbf{1}(\rho(\mathbf{x})\leq \Delta(\mathbf{x}))\right]\nonumber\\
&\leq& \frac{k-1}{N}\left(\mathbb{E}\left[\frac{1}{P^2(B(\mathbf{x},\rho(\mathbf{x})))}\right]\right)^\frac{1}{2}(\mathbb{E}[L^2 \rho^2(\mathbf{x})\mathbf{1}(\rho(\mathbf{x})\leq \Delta(\mathbf{x}))])^\frac{1}{2}\nonumber\\
&\overset{(a)}{\leq} &\frac{k-1}{N}\left(\frac{N(N-1)}{(k-1)(k-2)}\right)^\frac{1}{2}(\mathbb{E}[L^2 \rho^2(\mathbf{x})\mathbf{1}(\rho(\mathbf{x})\leq \Delta(\mathbf{x}))])^\frac{1}{2}\nonumber\\
&\overset{(b)}{\leq}& \sqrt{\frac{(k-1)(N-1)}{N(k-2)}}\left(\frac{L^2}{(mv_d)^\frac{2}{d}}\mathbb{E}[P^\frac{2}{d}(B(\mathbf{x},\rho(\mathbf{x})))]\right)^\frac{1}{2}\nonumber\\
&\overset{(c)}{\lesssim} &\left(\frac{k}{N}\right)^\frac{1}{d}.
\label{eq:I11}
\end{eqnarray}

Here, (a) uses the fact that $P(B(\mathbf{x},\rho(\mathbf{x})))$ follows Beta distribution $\text{Beta}(k,N-k+1)$. For (b), note that since the pdf is lower bounded by $m$, we have $P(B(\mathbf{x},\rho(\mathbf{x})))\geq mv_d\rho^d(\mathbf{x})$ if $\rho(\mathbf{x})\leq \Delta(\mathbf{x})$. For (c),  we use the following fact
\begin{eqnarray}
\mathbb{E}[P^\frac{2}{d}(B(\mathbf{x},\rho(\mathbf{x})))]&=&\frac{1}{\mathbb{B}(k,N-k+1)}\int u^\frac{2}{d} u^{k-1}(1-u)^{N-k}du\nonumber\\
&=&\frac{\Gamma\left(k+\frac{2}{d}\right)\Gamma(N+1)}{\Gamma\left(N+\frac{2}{d}+1\right)\Gamma(k)}\nonumber\\
&\lesssim &\left(\frac{k}{N}\right)^\frac{2}{d},
\end{eqnarray}
in which $\Gamma(x) = \int_0^\infty t^{x-1} e^{-t}dt$ and $\mathbb{B}(x,y) =\int t^{x-1} (1-t)^{y-1} dt$ are Gamma and Beta functions, respectively. 

If $\rho(\mathbf{x})>\Delta(\mathbf{x})$, since $m\leq f(\mathbf{x})\leq M$, 
\begin{eqnarray}
&&\mathbb{E}\left[\left|\frac{k-1}{NV(B(\mathbf{x},\rho(\mathbf{x})))}-\frac{k-1}{NP(B(\mathbf{x},\rho(\mathbf{x})))}f(\mathbf{x})\right|\mathbf{1}(\rho(\mathbf{x})>\Delta(\mathbf{x})\right]\nonumber\\
&\leq &\mathbb{E}\left[\frac{k-1}{NP(B(\mathbf{x},\rho(\mathbf{x})))} M\mathbf{1}(\rho(\mathbf{x})>\Delta(\mathbf{x}))\right]\nonumber\\
&\overset{(a)}{\leq} & M\mathbb{E}\left[\frac{k-1}{NP(B(\mathbf{x},\rho(\mathbf{x})))}\right]\text{P}(\rho(\mathbf{x})>\Delta(\mathbf{x}))\nonumber\\
&\lesssim & \text{P}(\rho(\mathbf{x})>\Delta(\mathbf{x})),
\label{eq:I12}
\end{eqnarray}
in which (a) holds because $1/P(B(\mathbf{x},\rho(\mathbf{x})))$ and $\mathbf{1}(\rho(\mathbf{x})>\Delta(\mathbf{x}))$ are negatively correlated.

Combining \eqref{eq:I11} and \eqref{eq:I12}, we have
\begin{eqnarray}
\mathbb{E}[|I_1|]\lesssim \left(\frac{k}{N}\right)^\frac{1}{d}+\text{P}(\rho(\mathbf{x})>\Delta(\mathbf{x})).
\label{eq:I1}
\end{eqnarray}

\textbf{Bound of $I_2$.}

Note that $\mathbb{E}[1/P(B(\mathbf{x},\rho(\mathbf{x})))]=N/(k-1)$, thus $\mathbb{E}[I_2]=0$. Therefore
\begin{eqnarray}
\mathbb{E}[I_2^2]&=&\Var[I_2]\nonumber\\
&=&\left(\frac{k-1}{N}\right)^2 f^2(\mathbf{x})\left(\mathbb{E}\left[\frac{1}{P^2(B(\mathbf{x},\rho(\mathbf{x})))}\right]-\left(\mathbb{E}\left[\frac{1}{P(B(\mathbf{x},\rho(\mathbf{x})))}\right]\right)^2\right)\nonumber\\
&=& \left(\frac{k-1}{N}\right)^2 f^2(\mathbf{x})\left(\frac{N(N-1)}{(k-1)(k-2)}-\frac{N^2}{(k-1)^2}\right)\nonumber\\
&=&f^2(\mathbf{x})\frac{N-k-1}{N(k-2)}.
\label{eq:I2}
\end{eqnarray}

Therefore $\mathbb{E}[|I_2|]\lesssim 1/\sqrt{k}$. Combining this with the bound of $I_1$ in \eqref{eq:I1}, we have 
\begin{eqnarray}
\mathbb{E}[|\hat{f}(\mathbf{x})-f(\mathbf{x})|]\lesssim \left(\frac{k}{N}\right)^\frac{1}{d}+\text{P}(\rho(\mathbf{x})>\Delta(\mathbf{x}))+k^{-\frac{1}{2}}.
\end{eqnarray}

Now integrate the above result over $\mathbf{x}\in S$. Define
\begin{eqnarray}
r_0=\left(\frac{k-1}{mv_dN}\right)^\frac{1}{d},
\end{eqnarray}
then $P(B(\mathbf{x},r_0))\geq (k-1)/N$. Hence, if $\Delta(\mathbf{x})>r_0$, 
\begin{eqnarray}
P(B(\mathbf{x},\Delta(\mathbf{x})))\geq mv_d\Delta^d(\mathbf{x})=mv_dr_0^d\left(\frac{\Delta(\mathbf{x})}{r_0}\right)^d=\frac{k-1}{N}\left(\frac{\Delta(\mathbf{x})}{r_0}\right)^d.
\end{eqnarray}
Then
\begin{eqnarray}
&&\int \text{P}(\rho(\mathbf{x})>\Delta(\mathbf{x}))\mathbf{1}(\Delta(\mathbf{x})>2^\frac{1}{d}r_0)d\mathbf{x}\nonumber\\
&\overset{(a)}{\leq}& \int \exp(-NP(B(\mathbf{x},\Delta(\mathbf{x}))))\left(\frac{eNP(B(\mathbf{x},\Delta(\mathbf{x})))}{k-1}\right)^{k-1} f(\mathbf{x})\mathbf{1}(\Delta(\mathbf{x})>2^\frac{1}{d}r_0)d\mathbf{x}\nonumber\\
&\overset{(b)}{\leq} & \int\exp\left[-(k-1)\left(\frac{\Delta(\mathbf{x})}{r_0}\right)^d\right]\left(e\left(\frac{\Delta(\mathbf{x})}{r_0}\right)^d\right)^{k-1}\mathbf{1}(\Delta(\mathbf{x})>2^\frac{1}{d}r_0)d\mathbf{x}\nonumber\\
&\overset{(c)}{\leq}& \int \exp\left[-(1-\ln 2)(k-1)\left(\frac{\Delta(\mathbf{x})}{r_0}\right)^d\right]d\mathbf{x}\nonumber\\
&\overset{(d)}{\leq}& V(S)\mathbb{E}\left[\exp\left[-(1-\ln 2)(k-1)\left(\frac{\Delta(U)}{r_0}\right)^d\right]\right]\nonumber\\
&=&V(S) \int_0^1 \text{P}\left(\exp\left[-(1-\ln 2)(k-1)\left(\frac{\Delta(U)}{r_0}\right)^d\right]>t\right)dt\nonumber\\
&=&V(S)\int_0^1 \text{P}\left(\Delta(U)<\left(\frac{\ln \frac{1}{t}}{(1-\ln 2)(k-1)}\right)^\frac{1}{d}r_0\right)dt\nonumber\\
&\overset{(e)}{\leq} & C_S\int_0^2 \frac{\ln^\frac{1}{d} t}{(1-\ln 2)^\frac{1}{d}(k-1)^\frac{1}{d}} r_0dt\nonumber\\
&=&\frac{C_S\Gamma\left(1+\frac{1}{d}\right)r_0}{(1-\ln 2)^\frac{1}{d}(k-1)^\frac{1}{d}}.
\end{eqnarray}

For (a), note that $\rho(\mathbf{x})>\Delta(\mathbf{x})$ is equivalent to the event that the number of samples in $B(\mathbf{x},\Delta(\mathbf{x}))$ is less than $k$. Therefore the probability can be bounded using Chernoff's inequality:
\begin{eqnarray}
\text{P}(\rho(\mathbf{x})>\Delta(\mathbf{x}))&=&\text{P}(n(\mathbf{x},\Delta(\mathbf{x}))<k-1)\nonumber\\
&\leq& \exp(-NP(B(\mathbf{x},\Delta(\mathbf{x}))))\left(\frac{eNP(B(\mathbf{x},\Delta(\mathbf{x})))}{k-1}\right)^{k-1},
\end{eqnarray}
in which $n(\mathbf{x},\Delta(\mathbf{x}))$ is the number of samples in $B(\mathbf{x},\Delta(\mathbf{x}))$, which follows a Binomial distribution with parameter $N$ and $P(B(\mathbf{x},\Delta(\mathbf{x})))$.

(b) uses the fact that $e^{-t}(et/(k-1))^{k-1}$ is monotonically increasing for $t>k-1$. (c) holds because $t-1-\ln t\geq (1-\ln 2)t$ for $t\geq 2$. In (d), $V(S)$ is the volume of the support $S$, and $U$ is a random variable following a uniform distribution in $S$. In (e), $C_S$ is the constant in Assumption \ref{ass:bounded} (c), which refers to the surface area of the support $S$. In addition,
\begin{eqnarray}
\int \text{P}(\rho(\mathbf{x})>\Delta(\mathbf{x}))\mathbf{1}(\Delta(\mathbf{x})\leq 2^\frac{1}{d} r_0)d\mathbf{x}\leq \int \mathbf{1}(\Delta(\mathbf{x})\leq 2^\frac{1}{d}r_0)d\mathbf{x}\leq 2^\frac{1}{d}r_0C_S.
\end{eqnarray}

Hence 
\begin{eqnarray}
\mathbb{E}\left[\norm{\hat{f}-f}_1\right]&\lesssim& \left(\frac{k}{N}\right)^\frac{1}{d}+\int \text{P}(\rho(\mathbf{x})>\Delta(\mathbf{x}))d\mathbf{x}+k^{-\frac{1}{2}}\nonumber\\
&\sim & \left(\frac{k}{N}\right)^\frac{1}{d}+k^{-\frac{1}{2}}.
\end{eqnarray}

The proof of the upper bound is now complete. The lower bound can be proved simply by standard minimax analysis in \cite{tsybakov2009introduction}. We now provide a simple proof. Find $2n$ points $\mathbf{a}_{i}$, $i=-n,-n+1,\ldots, -1, 1,\ldots, n$, such that $B(\mathbf{a}_i, r)\in S$ for any $i$, and $\norm{\mathbf{a}_j-\mathbf{a}_i}\geq 2r$ for any $j\neq i$. For $\mathbf{v}\in \mathbb{R}^d$, let
\begin{eqnarray}
f_\mathbf{v}(\mathbf{x}) = f_0(\mathbf{x})+v_i r g\left(\frac{\mathbf{x}-\mathbf{a}_i}{r}\right)-v_i r g\left(\frac{\mathbf{x}-\mathbf{a}_i}{r}\right),
\end{eqnarray}
in which
\begin{eqnarray}
f_0(\mathbf{x}) = 1/V(S)
\end{eqnarray}
is the pdf of the uniform distribution in support $S$ and 
\begin{eqnarray}
g(\mathbf{u}) = 1-\norm{\mathbf{u}}.
\end{eqnarray}

Then for any estimator $\hat{f}$,
\begin{eqnarray}
\underset{f\in \Sigma_A(S)}{\sup} \mathbb{E}\left[\norm{\hat{f}-f}_1\right] &\geq& \underset{\mathbf{v}\in \{-1,1\}^d}{\sup} \mathbb{E}\left[\norm{\hat{f}-f_\mathbf{v}}_1\right]\nonumber\\
&\geq &\mathbb{E}\left[\norm{\hat{f}-f_\mathbf{v}}_1\right]\nonumber\\
&=& \sum_{i=1}^n \mathbb{E}\left[\int_{B(\mathbf{a}_i,r)\cup B(\mathbf{a}_{-i}, r)} |\hat{f}-f_\mathbf{v}|d\mathbf{x}\right]\nonumber\\
&=& n\mathbb{E}\left[\int_{B(\mathbf{a}_1,r)\cap B(\mathbf{a}_{-1},r)} |\hat{f}-f_\mathbf{v}|d\mathbf{x}\right].
\end{eqnarray}
Let $\mathbf{v}_1 = (1,\ldots, 1)$, and $\mathbf{v}_2 = (-1,1,\ldots, 1)$, then from Le Cam's lemma \cite{tsybakov2009introduction}, we have

\begin{eqnarray}
\mathbb{E}\left[\int_{B(\mathbf{a}_1,r)\cup B(\mathbf{a}_{-1}, r)} |\hat{f}-f_\mathbf{v}|d\mathbf{x}\right]&\geq & \frac{1}{4} \norm{f_{\mathbf{v}_1}-f_{\mathbf{v}_2}}_1 e^{-ND(f_{\mathbf{v}_1}||f_{\mathbf{v}_2})}\nonumber\\
&\gtrsim & r^{d+1} e^{-Nr^{d+2}},
\end{eqnarray}
in which $D(\cdot||\cdot)$ is the KL divergence.
Hence, with $r\sim N^{-1/(d+2)}$,
\begin{eqnarray}
\underset{f\in \Sigma_A(S)}{\sup}\mathbb{E}\left[\norm{\hat{f}-f}_1\right]\gtrsim nr^{d+1}e^{-Nr^{d+2}} \sim N^{-\frac{1}{d+2}}.
\end{eqnarray}
The proof is complete.

\section{Proof of Theorem~\ref{thm:compact-unif}}\label{sec:bounded-unif}

Since $S$ is compact, there exists a constant $\mathcal{N}_0$, such that for sufficiently small $r$, the covering number of $S$ with balls with radius $r$ is bounded by $\mathcal{N}_0/r^d$. Therefore, we use $n$ balls with radius $r$ to cover the support set $S$, in which $n\leq \mathcal{N}_0/r^d$, and
\begin{eqnarray}
r=\min\left\{\left(\frac{k}{N}\right)^\frac{2}{d},k^{-\frac{1}{2}} \right\}.
\label{eq:rdef2}
\end{eqnarray} 
Denote $\mathbf{a}_1,\ldots, \mathbf{a}_n$ as the centers of these balls. For any $\epsilon>0$, define $\Delta(N,k)$ such that 
\begin{eqnarray}
\max\left\{D\left(\frac{k-1}{N}||\frac{k-1}{N}+\Delta(N,k)\right),D\left(\frac{k-1}{N}||\frac{k-1}{N}-\Delta(N,k)\right) \right\}=\frac{1}{N}\ln \frac{2n}{\epsilon},
\label{eq:Delta}
\end{eqnarray}
in which $D(p||q)=p\ln\frac{p}{q}+(1-p)\ln\frac{1-p}{1-q}$. Then we have the following lemma:
\begin{lem}\label{lem:Delta}
	If $k/N\rightarrow 0$ as $N\rightarrow \infty$, and $n<\frac{1}{2} e^{\frac{1}{8}(k-1)\epsilon}$, then
	\begin{eqnarray}
	\Delta(N,k)\leq 4\frac{k^{\frac{1}{2}}}{N}\sqrt{\ln \frac{2n}{\epsilon}}.
	\label{eq:Deltabound}
	\end{eqnarray}
\end{lem}
\begin{proof}
	Please see Appendix \ref{sec:Delta} for the proof.
\end{proof}

Now we provide a high probability bound of $P(B(\mathbf{x},\rho(\mathbf{x})))$. Denote $n(B(\mathbf{x},\rho(\mathbf{x})))$ as the number of samples in $B(\mathbf{x},\rho(\mathbf{x}))$, and define $r_0(\mathbf{x},p)$ such that $P(B(\mathbf{x},r_0(\mathbf{x},p)))=p$. Then
\begin{eqnarray}
\text{P}\left(P(B(\mathbf{x},\rho(\mathbf{x})))\geq \frac{k-1}{N}+\Delta(N,k)\right)&=&\text{P}\left(n\left(\mathbf{x},r_0\left(\mathbf{x},\frac{k-1}{N}+\Delta(N,k)\right)\right)\leq k-1\right)\nonumber\\
&\overset{(a)}{\leq} & \exp\left[-ND\left(\frac{k-1}{N}||\frac{k-1}{N}+\Delta(N,k)\right)\right]\nonumber\\
&\overset{(b)}{\leq}& \frac{\epsilon}{2n},
\label{eq:probub}
\end{eqnarray}
in which $n\left(\mathbf{x},r_0\left(\mathbf{x},\frac{k-1}{N}+\Delta(N,k)\right)\right)$ is the number of samples in $B\left(\mathbf{x},r_0\left(\mathbf{x},\frac{k-1}{N}+\Delta(N,k)\right)\right)$. From the definition of $r_0$, we have $P(B\left(\mathbf{x},r_0\left(\mathbf{x},\frac{k-1}{N}+\Delta(N,k)\right)\right))=(k-1)/N+\Delta(N,k)$. Hence, $n\left(\mathbf{x},r_0\left(\mathbf{x},\frac{k-1}{N}+\Delta(N,k)\right)\right)$ follows Binomial distribution with parameter $N$ and $(k-1)/N+\Delta(N,k)$. Then using Chernoff inequality, we get (a). Step (b) comes from \eqref{eq:Delta}.

Using similar arguments, we can also obtain
\begin{eqnarray}
\text{P}\left(P(B(\mathbf{x},\rho(\mathbf{x})))\leq \frac{k-1}{N}-\Delta(N,k)\right)\leq \frac{\epsilon}{2n}.
\label{eq:problb}
\end{eqnarray}
Using \eqref{eq:probub} and \eqref{eq:problb}, with probability at least $1-\epsilon$, we have
\begin{eqnarray}
\left|P(B(\mathbf{a}_i,\rho(\mathbf{x})))-\frac{k-1}{N}\right|<\Delta(N,k), \forall i\in\{1,\ldots,n \}.
\label{eq:cond}
\end{eqnarray}

In the remainder of this proof, we assume \eqref{eq:cond} is satisfied. We decompose $|\hat{f}(\mathbf{x})-f(\mathbf{x})|$ as following:

\begin{eqnarray}
\underset{\mathbf{x}\in S}{\sup}|\hat{f}(\mathbf{x})-f(\mathbf{x})|&\leq& \underset{\mathbf{x}\in S}{\sup} |\hat{f}(\mathbf{x})-\hat{f}(\mathbf{a}_i)|+\underset{i}{\max}|\hat{f}(\mathbf{a}_i)-f(\mathbf{a}_i)|+\underset{\mathbf{x}\in S}{\sup}|f(\mathbf{a}_i)-f(\mathbf{x})|\nonumber\\
&:=& I_1+I_2+I_3,
\label{eq:decomp}
\end{eqnarray}
in which $\mathbf{a}_i$ is the nearest point to $\mathbf{x}$ among $\{\mathbf{a}_1,\ldots,\mathbf{a}_n \}$.

We now bound these three terms separately.

\textbf{Bound of $I_1$.}
\begin{eqnarray}
|\hat{f}(\mathbf{x})-\hat{f}(\mathbf{a}_i)|&=&\left|\frac{k-1}{NV(B(\mathbf{x},\rho(\mathbf{x})))}-\frac{k-1}{NV(B(\mathbf{a}_i,\rho(\mathbf{a}_i)))}\right|\nonumber\\
&\leq &\frac{(k-1)M}{NP(B(\mathbf{a}_i,\rho(\mathbf{a}_i)))} \left|\frac{V(B(\mathbf{a}_i,\rho(\mathbf{a}_i))}{V(B(\mathbf{x},\rho(\mathbf{x})))}-1 \right|.
\end{eqnarray}
Here, $M$ is the constant in Assumption~\ref{ass:bounded} (a), which upper bounds $f(\mathbf{x})$ for all $\mathbf{x}\in S$. If \eqref{eq:cond} is satisfied, then for sufficiently large $N$,
\begin{eqnarray}
I_1&\leq & \frac{(k-1)M}{N\left(\frac{k-1}{N}-\Delta(N,k)\right)}\left|\frac{\rho^d(\mathbf{a}_i)}{(\rho(\mathbf{a}_i)-r)^d}-1\right|\nonumber\\
&\leq & 2M\left|\frac{1}{\left(1-\frac{r}{\rho(\mathbf{a}_i)}\right)^d}-1\right|.
\end{eqnarray}
According to the definition of $r$ in \eqref{eq:rdef2}, we have
\begin{eqnarray}
\frac{r}{\rho}&\leq & \left(\frac{Mv_d}{P(B(\mathbf{x},\rho))}\right)^\frac{1}{d} r\nonumber\\
&\leq & \left(\frac{Mv_d}{\frac{k-1}{N}-\Delta(N,k)}\right)^\frac{1}{d} r\nonumber\\
&\leq & \left(\frac{Mv_d}{\frac{k-1}{N}-\Delta(N,k)}\right)^\frac{1}{d}\left(\frac{k}{N}\right)^\frac{2}{d}.
\end{eqnarray}
Therefore there exists a constant $A_1$, such that
\begin{eqnarray}
I_1\lesssim\left(\frac{k}{N}\right)^\frac{1}{d}.
\label{eq:I1-rate}
\end{eqnarray}
\textbf{Bound of $I_2$.}
For all $\mathbf{x}\in S$,
\begin{eqnarray}
|\hat{f}(\mathbf{x})-f(\mathbf{x})|\leq \left|\frac{k-1}{NV(B(\mathbf{x},\rho))}-\frac{k-1}{NP(B(\mathbf{x},\rho))}f(\mathbf{x})\right|+\left|\frac{k-1}{NP(B(\mathbf{x},\rho))}-1\right|f(\mathbf{x}).
\label{eq:I20}
\end{eqnarray}

According to \eqref{eq:Deltabound}, if $k/\ln N\rightarrow \infty$ and $k/N\rightarrow 0$, for sufficiently large $N$, when \eqref{eq:cond} holds,
\begin{eqnarray}
\left|\frac{k-1}{NP(B(\mathbf{x},\rho))}-1\right|\lesssim k^{-\frac{1}{2}}\sqrt{\ln \frac{n}{\epsilon}}.
\label{eq:b1}
\end{eqnarray}
Moreover, under \eqref{eq:cond},
\begin{eqnarray}
\left|\frac{k-1}{NV(B(\mathbf{a}_i,\rho))}-\frac{k-1}{NP(B(\mathbf{a}_i,\rho))}f(\mathbf{a}_i)\right|&=&\frac{k-1}{NP(B(\mathbf{a}_i,\rho))}\left|\frac{P(B(\mathbf{a}_i,\rho))-f(\mathbf{a}_i)V(B(\mathbf{a}_i,\rho))}{V(B(\mathbf{a}_i,\rho))}\right|\nonumber\\
&\overset{(a)}{\leq}&\frac{k-1}{NP(B(\mathbf{a}_i,\rho))} L\rho \nonumber\\
&\overset{(b)}{\leq}&\frac{k-1}{NP(B(\mathbf{a}_i,\rho))} L(mv_d)^{-\frac{1}{d}}P^\frac{1}{d}(B(\mathbf{a}_i,\rho))\nonumber\\
&\leq &L(mv_d)^{-\frac{1}{d}} \frac{k-1}{N}\frac{1}{\left(\frac{k-1}{N}-\Delta(N,k)\right)^{1-\frac{1}{d}}}\nonumber\\
&\lesssim &\left(\frac{k}{N}\right)^\frac{1}{d}.
\label{eq:b2}
\end{eqnarray}
In (a), we use the H{\"o}lder assumption:
\begin{eqnarray}
|P(B(\mathbf{a}_i,\rho))-f(\mathbf{a}_i)V(B(\mathbf{a}_i,\rho))|&=& \left|\int_{B(\mathbf{a}_i,\rho)}(f(\mathbf{x})-f(\mathbf{a}))d\mathbf{x}\right|\nonumber\\
&\leq &\left|\int_{B(\mathbf{a}_i,\rho)}L\norm{\mathbf{x}-\mathbf{a}_i} d\mathbf{x}\right|\nonumber\\
&\leq & L\rho V(B(\mathbf{a}_i,\rho)).
\label{eq:Holder}
\end{eqnarray}
(b) uses the fact that $P(B(\mathbf{a}_i,\rho))\geq mv_d\rho^d$. 

Plugging \eqref{eq:b1} and \eqref{eq:b2} into \eqref{eq:I20}, we can show that as long as \eqref{eq:cond} holds, the following result holds for all $i=1,\ldots, n$:
\begin{eqnarray}
|\hat{f}(\mathbf{a}_i)-f(\mathbf{a}_i)|\lesssim k^{-\frac{1}{2}}\sqrt{\ln \frac{n}{\epsilon}}+\left(\frac{k}{N}\right)^\frac{1}{d}.
\label{eq:I2-rate0}
\end{eqnarray}
According to \eqref{eq:additional}, the additional assumption in Theorem \ref{thm:unbounded-unif}, it is possible to let
\begin{eqnarray}
n\leq \mathcal{N}_0/r^d\leq \mathcal{N}_0\max\left\{\left(\frac{N}{k}\right)^2,k^{\frac{d}{2}} \right\}.
\label{eq:nub}
\end{eqnarray} 
Hence, from \eqref{eq:I2-rate0} and \eqref{eq:nub},
\begin{eqnarray}
|\hat{f}(\mathbf{a}_i)-f(\mathbf{a}_i)|\lesssim k^{-\frac{1}{2}}\sqrt{\ln \frac{N}{\epsilon}}+\left(\frac{k}{N}\right)^\frac{1}{d}.
\label{eq:I2-rate}
\end{eqnarray}
\textbf{Bound of $I_3$.} According to assumption (b) and the definition of $r$ in \eqref{eq:rdef2},
\begin{eqnarray}
|f(\mathbf{x})-f(\mathbf{a}_i)|\leq L\underset{i}{\min} \norm{\mathbf{x}-\mathbf{a}_i}\leq Lr \lesssim k^{-\frac{1}{2}}.
\label{eq:I3-rate}
\end{eqnarray}
Recall that \eqref{eq:I1-rate}, \eqref{eq:I2-rate} and \eqref{eq:I3-rate} are all obtained under \eqref{eq:cond}, which holds with probability at least $1-\epsilon$. Based on these three equations, and use the upper bound of $n$ in \eqref{eq:nub}, we know that there exist two constants $C_1$ and $C_2$ such that
\begin{eqnarray}
|\hat{f}(\mathbf{x})-f(\mathbf{x})|\lesssim \left(\frac{k}{N}\right)^\frac{1}{d}+k^{-\frac{1}{2}}\sqrt{\ln \frac{N}{\epsilon}}
\end{eqnarray}
holds for all $\mathbf{x}\in S$ with probability at least $1-\epsilon$. The proof is complete.
\subsection{Proof of Lemma \ref{lem:Delta}}\label{sec:Delta}
From the definition of KL divergence, we have
\begin{eqnarray}
\frac{\partial^2 D(p||q)}{\partial q^2}=\frac{p}{q^2}-\frac{1-p}{(1-q)^2}.
\end{eqnarray}
If $\frac{1}{2}p<q<2p$, and $p$ is sufficiently small, we have
\begin{eqnarray}
\frac{\partial^2 D(p||q)}{\partial q^2}\geq \frac{p}{4p^2}-\frac{1-p}{(1-2p)^2}\geq \frac{1}{8p}.
\end{eqnarray}
Here we let $p=(k-1)/N$. Since $k/N\rightarrow 0$, for sufficiently large $N$, $p$ will be sufficiently small. Therefore
\begin{eqnarray}
\frac{\partial^2}{\partial q^2} D\left(\frac{k-1}{N}||q\right)\geq\frac{N}{8(k-1)} \geq \frac{N}{8k}
\end{eqnarray}
holds for $(k-1)/(2N)<q<2(k-1)/N$. Moreover, it can be shown that $\underset{p\rightarrow 0}{\lim} D(p||\frac{1}{2}p)/p=\ln 2-1/8>1/8$, and $\underset{p\rightarrow 0}{\lim} D(p||2p)/p=1-\ln 2>1/8$. Hence for sufficiently large $N$, $k/N$ is sufficiently small, we have
\begin{eqnarray}
\min\left\{ D\left(\frac{k-1}{N}||\frac{k-1}{2N}\right), D\left(\frac{k-1}{N}||\frac{2(k-1)}{N}\right) \right\}\geq \frac{k-1}{8N}.
\label{eq:KLb1}
\end{eqnarray}
According to the condition $n<\frac{1}{2}e^{\frac{1}{8}(k-1)\epsilon}$, we have
\begin{eqnarray}
\frac{1}{N}\ln \frac{2n}{\epsilon}<\frac{k-1}{8N}.
\label{eq:KLb2}
\end{eqnarray}

Therefore, using the second order Taylor expansion,
\begin{eqnarray}
D\left(\frac{k-1}{N}||\frac{k-1}{N}+\Delta(N,k)\right)&\overset{(a)}{=} &D\left(\frac{k-1}{N}||\frac{k-1}{N}\right)+\frac{1}{2}\left. \frac{\partial^2 D\left(\frac{k-1}{N}||q\right)}{\partial q^2}\right|_{q=\xi} \Delta^2 (N,k)\nonumber\\
&\overset{(b)}{\geq} &\frac{1}{2}\underset{\frac{k-1}{2N}<q<\frac{2(k-1)}{N}}{\inf} \frac{\partial^2 D\left(\frac{k-1}{N}||q\right)}{\partial q^2} \Delta^2 (N,k)\nonumber\\
&\geq & \frac{N}{16k}\Delta^2(N,k).
\end{eqnarray}
In (a), $\xi$ is in between $(k-1)/N$ and $(k-1)/N+\Delta(N,k)$. (b) holds because \eqref{eq:KLb1}, \eqref{eq:KLb2} and the definition of $\Delta(N,k)$ in \eqref{eq:Delta} imply that $(k-1)/N+\Delta(N,k)<2(k-1)/N$ and $(k-1)/N-\Delta(N,k)>(k-1)/(2N)$.

Similarly,
\begin{eqnarray}
D\left(\frac{k-1}{N}||\frac{k-1}{N}-\Delta(N,k)\right)\geq \frac{N}{16k}\Delta^2(N,k)
\end{eqnarray} 

 also holds. According to \eqref{eq:Delta}, we have
\begin{eqnarray}
\frac{N}{16k}\Delta^2(N,k)\leq \frac{1}{N}\ln \frac{2n}{\epsilon}.
\end{eqnarray}
Thus \eqref{eq:Deltabound} holds. The proof of Lemma \ref{lem:Delta} is complete.

Now we prove the corresponding minimax lower bound of the $\ell_\infty$ bound with unknown support, and show that no method is uniformly consistent. Let the distribution be one dimensional, $f_1(x)=1$ in $(0,1)$, and $f_2(x)=N/(N-1)$ in $(0,1-1/N)$. Use Le Cam's lemma \cite{tsybakov2009introduction},
\begin{eqnarray}
\underset{f}{\inf}\underset{f\in \Sigma}{\sup}\mathbb{E}\left[\norm{\hat{f}-f}_\infty\right]\geq \frac{1}{2}\norm{f_1-f_2}_\infty e^{-ND(f_2||f_1)}\geq \frac{1}{2}e^{-N\ln \frac{N}{N-1}}\rightarrow \frac{1}{2e}\neq 0.
\end{eqnarray}

On the contrary, if the support is known, then the minimax bound for known boundary has been derived in \cite{tsybakov2009introduction}.
\section{Proof of Theorem~\ref{thm:unbounded}}\label{sec:unbounded-l1}

Define
\begin{eqnarray}
f_+(\mathbf{x},r)=\underset{\mathbf{x}'\in B(\mathbf{x},r)}{\sup} f(\mathbf{x}').
\end{eqnarray}
We have the following two lemmas.
\begin{lem}\label{lem:pdf1}
	If $r\leq 1/(C_b(1-\ln f(\mathbf{x})))$, then
	\begin{eqnarray}
	f_+(\mathbf{x},r)\leq \frac{f(\mathbf{x})}{1-C_br(1-\ln f(\mathbf{x}))}.
	\end{eqnarray}
\end{lem}
\begin{proof}
	For all $u\in B(\mathbf{x},r)$, we have
	\begin{eqnarray}
	f(\mathbf{u})&=&f(\mathbf{x})+\nabla^T f(\xi)(\mathbf{u}-\mathbf{x})\nonumber\\
	&\leq & f(\mathbf{x})+\norm{\nabla f(\xi)}r\nonumber\\
	&\leq & f(\mathbf{x})+r\underset{\mathbf{v}\in B(\mathbf{x},r)}{\sup}\left[C_bf(\mathbf{v})\left(1+\ln \frac{1}{f(\mathbf{v})}\right)\right]\nonumber\\
	&\leq &f(\mathbf{x})+C_brf_+(\mathbf{x},r)\left(1+\ln \frac{1}{f_+(\mathbf{x},r)}\right)\nonumber\\
	&\leq & f(\mathbf{x})+C_b r f_+(\mathbf{x},r)\left(1+\ln \frac{1}{f(\mathbf{x})}\right).
	\end{eqnarray}
	Taking supremum over $u\in B(\mathbf{x},r)$, we have
	\begin{eqnarray}
	f_+(\mathbf{x},r)\leq f(\mathbf{x})+C_brf_+(\mathbf{x},r)\left(1+\ln \frac{1}{f(\mathbf{x})}\right).
	\end{eqnarray}
\end{proof}

\begin{lem}\label{lem:pdf2}
	If $r<1/(2C_b(1-\ln f(\mathbf{x})))$, then
	\begin{eqnarray}
	|P(B(\mathbf{x}))-f(\mathbf{x})V(B(\mathbf{x},r))|\leq C_cr^2 V(B(\mathbf{x},r))f(\mathbf{x})\left(1+\ln \frac{1}{f(\mathbf{x})}\right).
	\label{eq:pdfapprox}
	\end{eqnarray}
\end{lem}
\begin{proof}
	\begin{eqnarray}
	|P(B(\mathbf{x},r))-f(\mathbf{x})V(B(\mathbf{x},r))|&=&\left|\int_{V(B(\mathbf{x},r))}(f(\mathbf{u})-f(\mathbf{x}))d\mathbf{u}\right|\nonumber\\
	&=&\left|\int \left[\nabla^T f(\mathbf{x})(\mathbf{u}-\mathbf{x})+\frac{1}{2}(\mathbf{u}-\mathbf{x})^T \nabla^2 f(\xi)(\mathbf{u}-\mathbf{x})\right]d\mathbf{u}\right|\nonumber\\
	&\leq & \frac{1}{2}r^2 V(B(\mathbf{x},r))\underset{\mathbf{v}}{\sup}\norm{\nabla^2 f(\mathbf{v})}_{op}\nonumber\\
	&\leq & \frac{1}{2}r^2 V(B(\mathbf{x},r)) C_cf_+(\mathbf{x},r)\left(1+\ln \frac{1}{f_+(\mathbf{x},r)}\right).
	\end{eqnarray}
	According to Lemma \ref{lem:pdf1}, if $r<1/(2C_b(1-\ln f(\mathbf{x})))$, then $f_+(\mathbf{x},r)\leq 2f(\mathbf{x})$. Hence \eqref{eq:pdfapprox} holds.
\end{proof}

With Lemma \ref{lem:pdf1} and Lemma \ref{lem:pdf2}, we now prove the $\ell_1$ bound. Recall that the estimator is
\begin{eqnarray}
\hat{f}(\mathbf{x})=\left\{
\begin{array}{ccc}
\frac{k-1}{NV(B(\mathbf{x},\rho(\mathbf{x})))} &\text{if} & \rho(\mathbf{x})\leq a\nonumber\\
\frac{n(\mathbf{x},a)}{NV(B(\mathbf{x},a))} &\text{if} & \rho(\mathbf{x})>a.
\end{array}
\right.
\end{eqnarray}

Define
\begin{eqnarray}
I_1=\left\{
\begin{array}{ccc}
\frac{k-1}{NV(B(\mathbf{x},\rho(\mathbf{x})))}-\frac{k-1}{NP(B(\mathbf{x},\rho(\mathbf{x})))}f(\mathbf{x}) &\text{if} & \rho(\mathbf{x})\leq a\nonumber\\
\frac{n(\mathbf{x},a)}{NV(B(\mathbf{x},a))}-\frac{n(\mathbf{x},a)}{NP(B(\mathbf{x},a))}f(\mathbf{x})&\text{if} & \rho(\mathbf{x})>a,
\end{array}
\right.
\end{eqnarray}

and
\begin{eqnarray}
I_2=\left\{
\begin{array}{ccc}
\left(\frac{k-1}{NP(B(\mathbf{x},\rho(\mathbf{x})))}-1\right)f(\mathbf{x}) &\text{if} & \rho(\mathbf{x})\leq a\nonumber\\
\left(\frac{n(\mathbf{x},a)}{NP(B(\mathbf{x},a))}-1\right)f(\mathbf{x}) &\text{if} &\rho(\mathbf{x})>a
\end{array}.
\right.
\end{eqnarray}

With these definitions, we bound $\mathbb{E}[|I_1|]$ and $\mathbb{E}[|I_2|]$ for the following cases.

\textbf{Bound of $\mathbb{E}[|I_1|]$.} For the bound of $\mathbb{E}[|I_1|]$, our analysis considers the following three cases depending on the pdf value.

Case 1: $f(\mathbf{x})>4k/(Nv_da^d)$. In this case, $\ln f(\mathbf{x})\lesssim \ln N$. Therefore, if $a$ shrinks faster than $1/\ln N$, i.e. $a\ln N\rightarrow 0$ as $N\rightarrow \infty$, then for sufficiently large $N$,
\begin{eqnarray}
a<\min\left\{\frac{1}{2C_b(1-\ln f(\mathbf{x}))},\sqrt{\frac{1}{2C_c(1-\ln f(\mathbf{x}))}}\right\}.
\end{eqnarray}

Therefore, \eqref{eq:pdfapprox} holds for all $r\leq a$, which implies that if $r\leq a$, then
\begin{eqnarray}
|P(B(\mathbf{x},r))-f(\mathbf{x})V(B(\mathbf{x},r)))|&\leq &C_cr^2V(B(\mathbf{x},r)))f(\mathbf{x})(1-\ln f(\mathbf{x}))\nonumber\\
&\leq & \frac{1}{2}f(\mathbf{x})V(B(\mathbf{x},r)))(2C_c a^2(1-\ln f(\mathbf{x}))\nonumber\\
&\leq & \frac{1}{2}f(\mathbf{x})V(B(\mathbf{x},r)))\times 2C_c \frac{1}{2C_c(1-\ln f(\mathbf{x}))} (1-\ln f(\mathbf{x}))\nonumber\\
&=& \frac{1}{2}f(\mathbf{x})V(B(\mathbf{x},r)).
\end{eqnarray}
Hence for all $r\leq a$,
\begin{eqnarray}
P(B(\mathbf{x},r)))\geq \frac{1}{2}f(\mathbf{x})v_dr^d.
\label{eq:masslb}
\end{eqnarray}
In particular, let $r=a$, then
\begin{eqnarray}
P(B(\mathbf{x},a))\geq \frac{2k}{N},
\end{eqnarray}
and thus
\begin{eqnarray}
\text{P}(\rho(\mathbf{x})>a)&\leq& e^{-NP(B(\mathbf{x},a))}\left(\frac{eNP(B(\mathbf{x},a))}{k}\right)^k\nonumber\\
&\leq & e^{-(1-\ln 2)k}.
\label{eq:exceeda}
\end{eqnarray}
If $\rho(\mathbf{x})\leq a$, then
\begin{eqnarray}
|I_1|&=&\frac{k-1}{NP(B(\mathbf{x},\rho(\mathbf{x})))}\left|\frac{P(B(\mathbf{x},\rho(\mathbf{x})))}{V(B(\mathbf{x},\rho(\mathbf{x})))}-f(\mathbf{x})\right|\nonumber\\
&\leq & \frac{k-1}{NP(B(\mathbf{x},\rho(\mathbf{x})))}C_c\rho^2(\mathbf{x})f(\mathbf{x})(1-\ln f(\mathbf{x}))\nonumber\\
&\leq & \frac{k-1}{NP(B(\mathbf{x},\rho(\mathbf{x})))}C_cf(\mathbf{x})(1-\ln f(\mathbf{x}))\left(\frac{2P(B(\mathbf{x},\rho(\mathbf{x})))}{v_df(\mathbf{x})}\right)^\frac{2}{d}\nonumber\\
&\leq & 2^\frac{2}{d} C_c\frac{k-1}{N}(1-\ln f(\mathbf{x}))f^{1-\frac{2}{d}}(\mathbf{x})\text{P}^{\frac{2}{d}-1}(B(\mathbf{x},\rho(\mathbf{x}))).
\end{eqnarray}

If $\rho(\mathbf{x})>a$, then
\begin{eqnarray}
|I_1|&=&\left|\frac{n(\mathbf{x},a)}{V(B(\mathbf{x},a))}-\frac{n(\mathbf{x},a)}{P(B(\mathbf{x},a))}f(\mathbf{x})\right|\nonumber\\
&\leq &\frac{k-1}{V(B(\mathbf{x},a))}+\frac{k-1}{\frac{1}{2}f(\mathbf{x})v_da^d}f(\mathbf{x})\nonumber\\
&=&\frac{3(k-1)}{Nv_da^d}.
\end{eqnarray}

Hence
\begin{eqnarray}
\mathbb{E}[|I_1|]&\leq& 2^\frac{2}{d}C_c\frac{k-1}{N}f^{1-\frac{2}{d}}(\mathbf{x})(1-\ln f(\mathbf{x}))\mathbb{E}\left[P^{\frac{2}{d}-1}(B(\mathbf{x},\rho(\mathbf{x})))\right]+3\text{P}(\rho(\mathbf{x})>a)\frac{k-1}{Nv_da^d}\nonumber\\
&\leq & \left(\frac{k}{N}\right)^\frac{2}{d}f^{1-\frac{2}{d}}(\mathbf{x})(1-\ln f(\mathbf{x}))+f(\mathbf{x})\text{P}(\rho(\mathbf{x})>a)\nonumber\\
&\lesssim & \left(\frac{k}{N}\right)^\frac{2}{d}f^{1-\frac{2}{d}}(\mathbf{x})(1-\ln f(\mathbf{x}))+e^{-(1-\ln 2)k }f(\mathbf{x}),
\label{eq:I1-1}
\end{eqnarray}
in which the last step comes from \eqref{eq:exceeda}. Now we integrate \eqref{eq:I1-1} over all $\mathbf{x}$ such that $f(\mathbf{x})>4k/(Nv_da^d)$. We use the following lemma.

\begin{lem}
	(\cite{zhao2019minimax}, Lemma 6) If $\text{P}(f(\mathbf{X})<t)\leq C_d t^\beta$ for any $t>0$, then for any $p>0$ and any sequence $s_N\rightarrow 0$, 
	\begin{eqnarray}
	\int f^{1-p}(\mathbf{x})\mathbf{1}(f(\mathbf{x})>s_N)d\mathbf{x}\lesssim \left\{
	\begin{array}{ccc}
	1 &\text{if} & \beta>p\\
	\ln \frac{1}{s_N} &\text{if} & \beta=p\\
	s_N^{\beta-p} &\text{if} & \beta<p.
	\end{array}
	\right.
	\end{eqnarray}
\end{lem}

Therefore
\begin{eqnarray}
\int |I_1|\mathbf{1}\left(f(\mathbf{x})>\frac{4k}{Nv_da^d}\right) d\mathbf{x}&\lesssim & \left(\frac{k}{N}\right)^\frac{2}{d}\ln N\int f^{1-\frac{2}{d}}(\mathbf{x})\mathbf{1}\left(f(\mathbf{x})>\frac{4k}{Nv_da^d}\right) d\mathbf{x}\nonumber\\
&\lesssim &\left\{
\begin{array}{ccc}
\left(\frac{k}{N}\right)^\frac{2}{d}\ln N &\text{if} & \beta>\frac{2}{d}\nonumber\\
\left(\frac{k}{N}\right)^\beta\ln N a^{2-\beta d} &\text{if} & \beta<\frac{2}{d}\nonumber\\
\left(\frac{k}{N}\right)^\frac{2}{d}\ln^2 N &\text{if} & \beta=\frac{2}{d}.
\end{array}
\right.
\end{eqnarray}

Case 2: $1/N\leq f(\mathbf{x})\leq 4k/(Nv_da^d)$. Similar to Case 1, \eqref{eq:pdfapprox} still holds for all $r\leq a$ for sufficiently large $N$, as long as $a\ln N\rightarrow 0$. Hence \eqref{eq:masslb} still holds. If $\rho(\mathbf{x})\leq a$, then
\begin{eqnarray}
|I_1|&=&\left|\frac{k-1}{NV(B(\mathbf{x},\rho(\mathbf{x})))}-\frac{k-1}{NP(B(\mathbf{x},\rho(\mathbf{x})))}f(\mathbf{x})\right|\nonumber\\
&\leq & \frac{k-1}{NP(B(\mathbf{x},\rho(\mathbf{x})))}C_c\rho^2(\mathbf{x})f(\mathbf{x})(1-\ln f(\mathbf{x}))\nonumber\\
&\leq & \frac{k-1}{NP(B(\mathbf{x},\rho(\mathbf{x})))}C_c\rho^2(\mathbf{x})f(\mathbf{x})(1-\ln f(\mathbf{x})).
\end{eqnarray}

If $\rho(\mathbf{x})>a$, then
\begin{eqnarray}
|I_1|&=&\left|\frac{n(\mathbf{x},a)}{NV(B(\mathbf{x},a))}-\frac{n(\mathbf{x},a)}{NP(B(\mathbf{x},a))}f(\mathbf{x})\right|\nonumber\\
&\leq &\frac{n(\mathbf{x},a)}{NP(B(\mathbf{x},a))}C_ca^2 f(\mathbf{x})(1-\ln f(\mathbf{x}))\nonumber\\
&\leq& \frac{2n(\mathbf{x},a)}{Nv_d}C_ca^{2-d}(1-\ln f(\mathbf{x})).
\end{eqnarray}
Combining the bound for $\rho(\mathbf{x})\leq a$ and $\rho(\mathbf{x})>a$, we have
\begin{eqnarray}
\mathbb{E}[|I_1|]&=&\frac{k-1}{N}C_ca^2f(\mathbf{x})(1-\ln f(\mathbf{x}))\mathbb{E}\left[\frac{1}{P(B(\mathbf{x},a))}\mathbf{1}(\rho(\mathbf{x})\leq a)\right]\nonumber\\
&&\hspace{1cm}+\frac{2}{Nv_d}C_ca^{2-d}(1-\ln f(\mathbf{x}))\mathbb{E}[n(\mathbf{x},a)\mathbf{1}(\rho(\mathbf{x})>a)]\nonumber\\
&\leq & C_ca^2 f(\mathbf{x})(1-\ln f(\mathbf{x}))+\frac{2C_c}{v_d}a^{2-d}(1-\ln f(\mathbf{x}))P(B(\mathbf{x},a)).
\end{eqnarray}

Using similar steps used to show \eqref{eq:masslb}, we can also show that
\begin{eqnarray}
P(B(\mathbf{x},a))\leq \frac{3}{2} f(\mathbf{x})v_da^d.
\end{eqnarray}
Therefore
\begin{eqnarray}
\mathbb{E}[|I_1|]\lesssim f(\mathbf{x})(1-\ln f(\mathbf{x}))a^2,
\end{eqnarray}
and
\begin{eqnarray}
\int |I_1|\mathbf{1}\left(\frac{1}{N}\leq f(\mathbf{x})\leq\frac{4k}{Nv_da^d}\right) d\mathbf{x}&\lesssim & a^2\ln N\int f(\mathbf{x})\mathbf{1}\left(\frac{1}{N}\leq f(\mathbf{x})\leq\frac{4k}{Nv_da^d}\right) d\mathbf{x}\nonumber\\
&\lesssim & \left(\frac{k}{N}\right)^\beta a^{2-\beta d}\ln N.
\end{eqnarray}

Case 3: $f(\mathbf{x})<1/N$. Then
\begin{eqnarray}
\mathbb{E}[|I_1|]&=&\mathbb{E}\left[\left|\frac{k-1}{NV(B(\mathbf{x},\rho(\mathbf{x})))}-\frac{k-1}{NP(B(\mathbf{x},\rho(\mathbf{x})))}f(\mathbf{x})\right|\mathbf{1}(\rho(\mathbf{x})\leq a)\right]\nonumber\\
&& \hspace{1cm} + \mathbb{E}\left[\left|\frac{n(\mathbf{x},a)}{NV(B(\mathbf{x},a))}-\frac{n(\mathbf{x},a)}{NP(B(\mathbf{x},a))}f(\mathbf{x})\right|\mathbf{1}(\rho(\mathbf{x})>a)\right]\nonumber\\
&\leq & \frac{k-1}{Nv_d}\mathbb{E}\left[\frac{1}{\rho^d(\mathbf{x})}\mathbf{1}(\rho(\mathbf{x})\leq a)\right]+f(\mathbf{x})\mathbb{E}\left[\frac{k-1}{NP(B(\mathbf{x},\rho(\mathbf{x})))}\mathbf{1}(\rho(\mathbf{x})\leq a)\right]\nonumber\\
&&\hspace{1cm} + \frac{\mathbb{E}[n(\mathbf{x},a)]}{NV(B(\mathbf{x},a))}+\frac{\mathbb{E}[n(\mathbf{x},a)]}{NP(B(\mathbf{x},a))}f(\mathbf{x})\nonumber\\
&\leq & \frac{k-1}{Nv_d}\mathbb{E}\left[\frac{1}{\rho^d(\mathbf{x})}\mathbf{1}(\rho(\mathbf{x})\leq a)\right]+2f(\mathbf{x})+\frac{P(B(\mathbf{x},a))}{V(B(\mathbf{x},a))}.
\label{eq:I1-3-1}
\end{eqnarray}
In the last step, we use $\mathbb{E}[1/P(B(\mathbf{x},\rho(\mathbf{x})))]=N/(k-1)$ and $\mathbb{E}[n(\mathbf{x},a)]=NP(B(\mathbf{x},a))$.

Now we bound $P(B(\mathbf{x},a))$. Recall that $f_+(\mathbf{x})=\underset{\mathbf{x}'}{\sup} f(\mathbf{x}')$. According to assumption (b),
\begin{eqnarray}
\frac{\partial f_+(\mathbf{x},r)}{\partial r}&\leq& \underset{\mathbf{x}'\in B(\mathbf{x},r)}{\sup}\norm{\nabla f(\mathbf{x}')}\nonumber\\
&\leq & \underset{\mathbf{x}'\in B(\mathbf{x},r)}{\sup}C_bf(\mathbf{x})\left(1+\ln \frac{1}{f(\mathbf{x})}\right)\nonumber\\
&\leq & C_b f_+(\mathbf{x},r)\left(1+\ln \frac{1}{f_+(\mathbf{x},r)}\right).
\end{eqnarray}

Define a function $g(r)$, such that
\begin{eqnarray}
g(0)=f(\mathbf{x}),g'(r)=C_bg(r)(1-\ln g(r)),
\end{eqnarray}
then $f_+(\mathbf{x},r)\leq g(r)$. It can be solved that
\begin{eqnarray}
g(r)=e^{1-C_b r}f^{e^{-C_b r}}(\mathbf{x}).
\end{eqnarray}
Therefore 
\begin{eqnarray}
f_+(\mathbf{x},a)\leq ef^{e^{-C_ba}}(\mathbf{x}),
\end{eqnarray}
and
\begin{eqnarray}
P(B(\mathbf{x},a))\leq ef^{e^{-C_b a}}v_d a^d.
\end{eqnarray}
Thus the first term in the right hand side of \eqref{eq:I1-3-1} can be bounded by
\begin{eqnarray}
\mathbb{E}\left[\frac{1}{\rho^d(\mathbf{x})}\mathbf{1}(\rho(\mathbf{x})\leq a)\right]&\leq & \int_{\frac{1}{a^d}}^\infty \text{P}(\rho(\mathbf{x})<t^{-\frac{1}{d}})dt\nonumber\\
&\leq & \int_{\frac{1}{a^d}}^\infty \left(\frac{eNP(B(\mathbf{x},t^{-\frac{1}{d}}))}{k}\right)^k dt\nonumber\\
&\leq & \int_{\frac{1}{a^d}}^\infty \left(\frac{eN}{k}\right)^k e^k f^{ke^{-C_bt^{-\frac{1}{d}}}}(\mathbf{x})\left(\frac{v_d}{t}\right)^k dt\nonumber\\
&\leq & \int_{\frac{1}{a^d}}^\infty \left(\frac{eN}{k}\right)^k e^k f^{ke^{-C_ba}}(\mathbf{x})\left(\frac{v_d}{t}\right)^k dt\nonumber\\
&\leq &\int_{\frac{1}{a^d}}^\infty \left(\frac{eN}{k}\right)^k e^k f^{k(1-C_b a)}(\mathbf{x})\left(\frac{v_d}{t}\right)^k dt\nonumber\\
&\leq & \left(\frac{eN}{k}\right)^k e^k f^{k(1-C_ba)}(\mathbf{x})v_d^k \frac{1}{k-1}(a^d)^{k-1}\nonumber\\
&=&\left(\frac{e^2 N}{k} f^{1-C_b a\frac{k}{k-1}}(\mathbf{x}) v_da^d \right)^{k-1} f(\mathbf{x})\frac{e^2 N}{k} v_d.
\end{eqnarray}

For arbitrarily small $\delta$, since $f(\mathbf{x})<1/N$ and $a\rightarrow 0$ as $N\rightarrow \infty$, for sufficiently large $N$, we have
\begin{eqnarray}
f^{1-C_ba\frac{k}{k-1}}(\mathbf{x})\leq \left(\frac{1}{N}\right)^{1-\delta}.
\end{eqnarray}
Therefore
\begin{eqnarray}
\left(\frac{e^2 N}{k} f^{1-C_b a\frac{k}{k-1}}(\mathbf{x}) v_da^d \right)^{k-1}\leq \left(\frac{N^\delta a^d}{k}\right)^{k-1}.
\end{eqnarray}
Pick $\delta$ such that $N^\delta a^d<k/2$ for sufficiently large $N$, then
\begin{eqnarray}
\mathbb{E}\left[\frac{1}{\rho^d(\mathbf{x})}\mathbf{1}(\rho(\mathbf{x})\leq a)\right]\leq 2^{-k} f(\mathbf{x})\frac{e^2 N}{v_d}.
\end{eqnarray}
Hence
\begin{eqnarray}
\mathbb{E}[|I_1|]\leq \frac{k-1}{Nv_d}2^{-k} f(\mathbf{x})\frac{e^2 N}{k}v_d+2f(\mathbf{x})+ef^{e^{-C_ba}}(\mathbf{x}),
\end{eqnarray}
and for arbitrarily small $\delta>0$,
\begin{eqnarray}
\int |I_1|\mathbf{1}\left( f(\mathbf{x})<\frac{1}{N}\right)d\mathbf{x}\lesssim N^{-\beta+\delta}.
\end{eqnarray}

\textbf{Bound of $\mathbb{E}[|I_2|]$.}
\begin{eqnarray}
&&\mathbb{E}[|I_2|]= \mathbb{E}\left[\left|\frac{k-1}{NP(B(\mathbf{x},\rho(\mathbf{x})))}-1\right|f(\mathbf{x})\mathbf{1}(\rho(\mathbf{x})\leq a) \right]\nonumber\\&&+\mathbb{E}\left[\left|\frac{n(\mathbf{x},a)}{NP(B(\mathbf{x},a))}-1\right|f(\mathbf{x})\mathbf{1}(\rho(\mathbf{x})>a)\right].
\end{eqnarray}
For the first term, \eqref{eq:I2} still holds, which yields
\begin{eqnarray}
\mathbb{E}\left[\left|\frac{k-1}{NP(B(\mathbf{x},\rho(\mathbf{x})))}-1\right|f(\mathbf{x})\mathbf{1}(\rho(\mathbf{x})\leq a) \right]\lesssim f(\mathbf{x})k^{-\frac{1}{2}}.
\end{eqnarray}
Therefore it remains to bound the second term. We consider different cases depending on the value of $f(\mathbf{x})$.

Case 1: $f(\mathbf{x})>4k/(Nv_da^d)$. Then from \eqref{eq:exceeda}, we have
\begin{eqnarray}
\mathbb{E}\left[\left|\frac{n(\mathbf{x},a)}{NP(B(\mathbf{x},a))}-1\right|f(\mathbf{x})\mathbf{1}(\rho(\mathbf{x})>a)\right]&\leq& \left( \frac{k-1}{Nv_da^d}+1\right)f(\mathbf{x})\text{P}(\rho(\mathbf{x})>a)\nonumber\\
&\leq & 2f(\mathbf{x})e^{-(1-\ln 2)k}.
\end{eqnarray}

Case 2: $1/(Na^d)\leq f(\mathbf{x})\leq 4k/(Nv_da^d)$. Note that now the lower threshold of case 2 is different than that used in the proof of the bound of $\mathbb{E}[|I_1|]$. Then
\begin{eqnarray}
\mathbb{E}\left[\left|\frac{n(\mathbf{x},a)}{NP(B(\mathbf{x},a))}-1\right|f(\mathbf{x})\mathbf{1}(\rho(\mathbf{x})>a)\right]&\leq & \sqrt{\mathbb{E}\left[\left(\frac{n(\mathbf{x},a)}{NP(B(\mathbf{x},a))}-1\right)^2\right]} f(\mathbf{x})\nonumber\\
&=&\frac{1}{NP(B(\mathbf{x},a))}\sqrt{\Var[n(\mathbf{x},a)]}f(\mathbf{x})\nonumber\\
&\leq & \frac{1}{\sqrt{NP(B(\mathbf{x},a))}}f(\mathbf{x})\nonumber\\
&\leq & \frac{f(\mathbf{x})}{\sqrt{\frac{1}{2}Nf(\mathbf{x})v_da^d}} f(\mathbf{x})\nonumber\\
&\lesssim & \frac{f^{\frac{1}{2}}(\mathbf{x})}{N^\frac{1}{2} a^\frac{d}{2}}.
\end{eqnarray}

Case 3: $f(\mathbf{x})<1/(Na^d)$. Then
\begin{eqnarray}
\mathbb{E}\left[\left|\frac{n(\mathbf{x},a)}{NP(B(\mathbf{x},a))}-1\right|f(\mathbf{x})\mathbf{1}(\rho(\mathbf{x})>a)\right]&\leq& \mathbb{E}\left[\frac{n(\mathbf{x},a)}{NP(B(\mathbf{x},a))}+1\right]f(\mathbf{x})\nonumber\\
&=& 2f(\mathbf{x}).
\end{eqnarray}

Therefore
\begin{eqnarray}
\int \mathbb{E}[|I_2|]d\mathbf{x}&\lesssim& k^{-\frac{1}{2}}\int f(\mathbf{x})d\mathbf{x}+N^{-\frac{1}{2}}a^{-\frac{d}{2}}\int f^{\frac{1}{2}}(\mathbf{x})\mathbf{1}\left(\frac{1}{Na^d}\leq f(\mathbf{x})\leq \frac{4k}{Nv_da^d}\right)d\mathbf{x}\nonumber\\
&&\hspace{1cm} + \int f(\mathbf{x})\mathbf{1}\left(f(\mathbf{x})<\frac{1}{Na^d}\right)d\mathbf{x}\nonumber\\
&\lesssim & \left(\frac{1}{Na^d}\right)^\beta \ln N +N^{-\frac{1}{2}}a^{-\frac{d}{2}}+k^{-\frac{1}{2}}.
\end{eqnarray}

The $\ell_1$ error of the density estimate with the truncated kNN estimator is bounded by
\begin{eqnarray}
\mathbb{E}\left[\norm{\hat{f}-f}_1\right]&\leq& \int \mathbb{E}[|I_1|]d\mathbf{x}+\int \mathbb{E}[|I_2|]d\mathbf{x}\nonumber\\
&\lesssim & \left(\frac{k}{N}\right)^\frac{2}{d}\ln^2 N+\left(\frac{k}{N}\right)^\beta a^{2-\beta d}\ln N+\left(\frac{1}{Na^d}\right)^\beta \ln N +N^{-\frac{1}{2}}a^{-\frac{d}{2}}+k^{-\frac{1}{2}}.\nonumber\\
\end{eqnarray}

If $\beta\leq 1/2$, let
\begin{eqnarray}
a\sim N^{-\frac{\beta^2}{d\beta^2+1}}, k\sim N^{\frac{2\beta}{d\beta^2+1}},
\end{eqnarray}
we have
\begin{eqnarray}
\mathbb{E}\left[\norm{\hat{f}-f}_1\right]\lesssim N^{-\frac{\beta}{d\beta^2 +1}}\ln^2 N.
\end{eqnarray}
If $\beta>1/2$, let
\begin{eqnarray}
a\sim N^{-\frac{1}{d+4}}, k\sim N^{-\frac{4}{d+4}},
\end{eqnarray}
we have
\begin{eqnarray}
\mathbb{E}[\norm{\hat{f}-f}_1]\lesssim N^{-\frac{2}{d+4}}\ln ^2 N.
\end{eqnarray}
The proof is complete.

\section{Proof of Theorem~\ref{thm:unbounded-lb}}\label{sec:minimax}

Define $f_0(\mathbf{x})$ such that
\begin{eqnarray}
f_0(\mathbf{x})=\left\{
\begin{array}{ccc}
\frac{1}{N} &\text{if} & \norm{\mathbf{x}}<r\nonumber\\
\frac{1}{2v_dR^d} &\text{if} & \norm{\mathbf{x}-\mathbf{c}}<R,
\end{array}
\right.
\end{eqnarray}
in which $R$ is fixed and $r=N^{\frac{1-\beta}{d}}$. $\norm{\mathbf{c}}$ is sufficiently large, so that $B(\mathbf{0},r)$ and $B(\mathbf{c},R)$ do not intersect. For other $\mathbf{x}$, i.e. for $\mathbf{x}\notin B(\mathbf{0},r)\cup B(\mathbf{c},R)$, $f_0$ is designed such that $f_0$ satisfies assumptions (a)-(d) with constant $C_b$, $C_c$ and $C_d/2$. 

Let $g(\mathbf{x})$ be a function supported in $B(\mathbf{0},1)$, with $\norm{g}_\infty\leq g_m$, in which
\begin{eqnarray}
g_m=\frac{\ln 2}{32v_d\ln 3},
\end{eqnarray}
and
\begin{eqnarray}
\norm{\nabla^2 g(\mathbf{x})}_{op}\leq \frac{1}{2}C_b.
\end{eqnarray} 

The above constructions are possible for sufficiently large $C_b$, $C_c$ and $C_d$. Find $\mathbf{a}_i$, $i=-n,-(n-1),\ldots, -1,1,\ldots,n$, such that $B(\mathbf{a}_i,1)$ are mutually disjoint, and $B(\mathbf{a}_i,1)\subset B(\mathbf{x},r)$ for all $i$. Define
\begin{eqnarray}
f_\mathbf{v}(\mathbf{x})=f_0(\mathbf{x})+\frac{v_i}{N}g(\mathbf{x}-\mathbf{a}_i)-\frac{v_i}{N}g(\mathbf{x}-\mathbf{a}_{-i}),
\end{eqnarray}
in which $\mathbf{v}\in \{-1,1\}^d$. 

According to Varshamov-Gilbert Lemma \cite{gilbert1952comparison}, there exists $N_G$ elements $\mathbf{v}^{(j)}$, $j=1,\ldots, N_G$, $N_G\geq 2^{n/8}$, such that $H(\mathbf{v}^{(j)},\mathbf{v}^{(k)})\geq n/8$ for all $0\leq j<k<N_G$, in which $H$ is the Hamming distance. Denote
\begin{eqnarray}
\mathcal{V}=\{\mathbf{v}^{(j)},j=1,\ldots,N_G \}.
\end{eqnarray}
Then the KL divergence between $f_{\mathbf{v}^{(j)}}$ and $f_{\mathbf{v}^{(k)}}$ is bounded by
\begin{eqnarray}
D(f_{\mathbf{v}^{(j)}}||f_{\mathbf{v}^{(k)}})&\leq& H(\mathbf{v}^{(i)},\mathbf{v}^{(j)})\int_{B(\mathbf{a}_i,1)\cup B(\mathbf{a}_{-i},1)}\left(f_0(\mathbf{x})+\frac{1}{N}g(\mathbf{x}-\mathbf{a}_i)-\frac{1}{N}g(\mathbf{x}-\mathbf{a}_{-i})\right) \nonumber\\
&&\hspace{1cm} \ln \frac{f_0(\mathbf{x})+\frac{1}{N}g(\mathbf{x}-\mathbf{a}_i)-\frac{1}{N} g(\mathbf{x}-\mathbf{a}_{-i})}{f_0(\mathbf{x})-\frac{1}{N}g(\mathbf{x}-\mathbf{a}_i)+\frac{1}{N}g(\mathbf{x}-\mathbf{a}_{-i})}d\mathbf{x}\nonumber\\
&\leq & H(\mathbf{v}^{(i)},\mathbf{v}^{(j)})\left[\int_{B(\mathbf{a}_i,1)\cup B(\mathbf{a}_{-1},1)}\frac{1}{N}(g(\mathbf{x}-\mathbf{a}_{i})-g(\mathbf{x}-\mathbf{a}_{-i}))\ln 3d\mathbf{x}\right]\nonumber\\
&\leq & 2\ln 3\frac{v_dg_m}{N}H(\mathbf{v}^{(i)},\mathbf{v}^{(j)}).
\end{eqnarray}
Since we have $N$ samples, denote $P_{\mathbf{v}^{(j)}}$ as the joint distribution of these $N$ samples, then
\begin{eqnarray}
D(P_{\mathbf{v}^{(j)}}||P_{\mathbf{v}^{(k)}})\leq 2\ln 3 v_dg_mH(\mathbf{v}^{(j)},\mathbf{v}^{(k)})\leq 2\ln 3 nv_d g_m=\frac{1}{16}n\ln 2.
\end{eqnarray}
Define
\begin{eqnarray}
\hat{\mathbf{v}}=\underset{\mathbf{v}}{\arg\min}\norm{\hat{f}-f_\mathbf{v}}_1.
\end{eqnarray}
Let $\mathbf{V}$ be a random variable that is uniformly distributed in $\mathcal{V}$, and the corresponding estimate is $\hat{\mathbf{V}}$, then from Fano's inequality,
\begin{eqnarray}
\underset{v}{\sup}\text{P}(\hat{\mathbf{V}}\neq \mathbf{V})\geq 1-\frac{\underset{j,k}{\max} D(P_{\mathbf{v}^{(j)}}||P_{\mathbf{v}^{(k)}})+\ln 2}{\ln N_G}\geq 1-\frac{\frac{1}{16}n\ln 2+\ln 2}{\frac{n}{8}\ln 2}.
\end{eqnarray}
For sufficiently large $N$,
\begin{eqnarray}
\text{P}(\hat{\mathbf{V}}\neq \mathbf{V})\geq \frac{1}{3}.
\end{eqnarray}
Note that if $\hat{\mathbf{V}}\neq \mathbf{V}$, then
\begin{eqnarray}
\norm{\hat{f}-f_\mathbf{V}}_1&\geq& \frac{1}{2}\norm{f_{\hat{\mathbf{V}}}-f_\mathbf{V}}_1\nonumber\\
&\geq & \frac{1}{2}H(\hat{\mathbf{V}},\mathbf{V})\times 4\int g(\mathbf{x})d\mathbf{x}\nonumber\\
&\geq & \frac{n}{4N}\int g(\mathbf{x})d\mathbf{x}.
\end{eqnarray}

To satisfy the assumptions, the maximum $n$ we can take is $n\sim r^d\sim N^{1-\beta}$. Then
\begin{eqnarray}
\mathbb{E}\left[\norm{\hat{f}-f_\mathbf{V}}_1\right]\geq \frac{1}{3}\frac{n}{4N}\int g(\mathbf{x})d\mathbf{x}\gtrsim N^{-\beta}.
\end{eqnarray}

Moreover, from the standard minimax analysis in \cite{tsybakov2009introduction}, it can be proved that
\begin{eqnarray}
\mathbb{E}\left[\norm{\hat{f}-f_\mathbf{V}}_1\right]\gtrsim N^{-\frac{2}{d+4}}.
\end{eqnarray}
Combine these two bounds, the proof of the minimax lower bound of density estimation with $\ell_1$ criterion is complete.

\section{Proof of Proposition \ref{prop}}\label{sec:kde}

In this appendix, we show a lower bound of the $\ell_1$ estimation error of the kernel density estimator. Recall that the kernel density estimator is defined as
\begin{eqnarray}
\hat{f}(\mathbf{x})=\frac{1}{Nh^d}\sum_{i=1}^NK\left(\frac{\mathbf{X}_i-\mathbf{x}}{h}\right),
\end{eqnarray}

in which $\int K(\mathbf{u})d\mathbf{u}=1$. For simplicity, we assume that $K$ is supported in $B(\mathbf{0},1)$.

Firstly, 
\begin{eqnarray}
\mathbb{E}\left[\norm{\hat{f}-f}_1\right]&=&\int \mathbb{E}[|\hat{f}(\mathbf{x})-f(\mathbf{x})|]d\mathbf{x}\nonumber\\
&\geq & \int |\mathbb{E}[\hat{f}(\mathbf{x})]-f(\mathbf{x})|d\mathbf{x}\nonumber\\
&=&\norm{f\star K_h-f}_1,
\label{eq:kde1}
\end{eqnarray}
in which $\star$ means convolution and $K_h(\cdot)=K(\cdot/h)/h^d$. $f\star K_h(\mathbf{x})$ is a weighted average of pdf in $B(\mathbf{x},h)$. Then there are many ways to construct $f$ so that
\begin{eqnarray}
\norm{f\star K_h-f}_1\gtrsim h^2.
\label{eq:kde2}
\end{eqnarray} 
We omit the detailed construction for simplicity. Moreover, define
\begin{eqnarray}
f_0(\mathbf{x})=\left\{
\begin{array}{ccc}
\frac{1}{Nv_dh^d} &\text{if} & \norm{\mathbf{x}}<r\\
\frac{1}{2v_dR^d} &\text{if} & \norm{\mathbf{x}-\mathbf{c}}<R,
\end{array}
\right.
\end{eqnarray}
in which $\norm{\mathbf{c}}$ is sufficiently large so that $B(\mathbf{0},r)$ and $B(\mathbf{c},R)$ do not intersect.

In order to ensure that $f_0(\mathbf{x})$ satisfies assumption (d), we set
\begin{eqnarray}
r=(Nv_dh^d)^{\frac{1-\beta}{d}},
\label{eq:r}
\end{eqnarray}
and for $\mathbf{x}\notin B(\mathbf{0},r)\cup B(\mathbf{c},R)$, $f_0$ is constructed so that assumptions (a)-(d) are satisfied.

If $B(\mathbf{x},h)\subset B(\mathbf{0},r)$, denote $n(\mathbf{x},h)$ as the number of samples in $B(\mathbf{x},h)$, then
\begin{eqnarray}
\text{P}(\hat{f}(\mathbf{x})=0)=\text{P}(n(\mathbf{x},h)=0)=\left(1-\frac{1}{N}\right)^N\rightarrow e^{-1} \text{ as } N\rightarrow \infty.
\end{eqnarray}

Thus for all $\mathbf{x}$ such that $B(\mathbf{x},h)\in B(\mathbf{0},r)$, i.e. $\mathbf{x}\in B(\mathbf{0},r-h)$,
\begin{eqnarray}
\mathbb{E}[|\hat{f}(\mathbf{x})-f(\mathbf{x})|]\geq \text{P}(\hat{f}(\mathbf{x})=0)f(\mathbf{x})=e^{-1}f(\mathbf{x}),
\end{eqnarray}
and
\begin{eqnarray}
\mathbb{E}\left[\int |\hat{f}(\mathbf{x})-f(\mathbf{x})|d\mathbf{x}\right]\geq \int_{B(\mathbf{0},r-h)}e^{-1} f(\mathbf{x})d\mathbf{x}=\frac{1}{Nev_dh^d}v_d(r-h)^d.
\end{eqnarray}
From \eqref{eq:r}, for sufficiently large $N$, $h<r/2$, hence
\begin{eqnarray}
\mathbb{E}\left[\norm{\hat{f}-f}_1\right]\gtrsim (Nh^d)^{-\beta}.
\label{eq:kde3}
\end{eqnarray}
Combining \eqref{eq:kde1}, \eqref{eq:kde2} and \eqref{eq:kde3}, we have
\begin{eqnarray}
\underset{f\in \Sigma_B}{\sup}\mathbb{E}\left[\norm{\hat{f}-f}_1\right]\gtrsim (Nh^d)^{-\beta}+h^2,
\end{eqnarray}
thus
\begin{eqnarray}
\underset{h}{\inf}\underset{f\in \Sigma_B}{\sup}\mathbb{E}\left[\norm{\hat{f}-f}_1\right]\gtrsim N^{-\frac{2\beta}{2+d\beta}}.
\end{eqnarray}

Moreover, the minimax lower bound is
\begin{eqnarray}
\underset{\hat{f}}{\inf}\underset{f\in \Sigma_B}{\sup}\mathbb{E}\left[\norm{\hat{f}-f}_1\right]\gtrsim N^{-\min\left\{\frac{2}{d+4},\beta\right\}}.
\end{eqnarray}

Kernel density estimator can not have a better convergence rate than the minimax lower bound. Therefore
\begin{eqnarray}
\underset{h}{\inf}\underset{f\in \Sigma_B}{\sup}\mathbb{E}\left[\norm{\hat{f}-f}_1\right]\gtrsim N^{-\min\left\{\frac{2\beta}{2+d\beta},\frac{2}{d+4}\right\}}.
\end{eqnarray}

\section{Proof of Theorem \ref{thm:unbounded-unif}}\label{sec:unbounded-unif}

Despite that for the $\ell_\infty$ error we use a kNN density estimator without truncation, for analysis convenience, we still define $a$ such that
\begin{eqnarray}
a&=&\min\left\{\frac{1}{2C_b(1-\ln f_c)}, \sqrt{\frac{1}{2C_c(1-\ln f_c)}}\right\},\\
f_c&=&\frac{4k}{Nv_da^d}.
\end{eqnarray}
This construction ensures that if $f(\mathbf{x})\geq f_c$, then Lemma \ref{lem:pdf2} holds for all $r\leq a$. Define
\begin{eqnarray}
S=\left\{\mathbf{x}|f(\mathbf{x})>f_c \right\},
\end{eqnarray}
and divide $S$ into two parts:
\begin{eqnarray}
S_1&=& \{\mathbf{x}|B(\mathbf{x},h)\subset S \},\label{eq:S1def}\\
S_2&=& S\setminus S_1,\label{eq:S2def}
\end{eqnarray}
in which
\begin{eqnarray}
h=\min\left\{\left(\frac{1}{16}\right)^\frac{1}{d},\frac{1}{2}\right\}a.
\label{eq:h}
\end{eqnarray}
We provide the uniform bound of the estimation error within $S$ and $S^c$ separately.

\textbf{Bound in $S$.} Similar to the case with bounded support, find $\mathbf{a}_1,\ldots,\mathbf{a}_n$, such that $\cup B(\mathbf{a}_i,r)$ covers $S$. Define $\Delta(N,k)$ such that
\begin{eqnarray}
\max\left\{D\left(\frac{k-1}{N}||\frac{k-1}{N}+\Delta(N,k)\right), D\left(\frac{k-1}{N}||\frac{k-1}{N}-\Delta(N,k)\right) \right\}=\frac{1}{N}\ln \frac{4n}{\epsilon}.
\end{eqnarray}
Then follow steps in the proof for distributions with bounded support, with probability at least $1-\epsilon/2$,
\begin{eqnarray}
\left|P(B(\mathbf{a}_i,\rho))-\frac{k-1}{N}\right|<\Delta(N,k),
\label{eq:proberr}
\end{eqnarray}
for all $i=1,\ldots,n$. Similar to Lemma \ref{lem:Delta}, it can be shown that
\begin{eqnarray}
\Delta(N,k)\leq 4\frac{k^{\frac{1}{2}}}{N}\sqrt{\ln \frac{4n}{\epsilon}}.
\label{eq:Deltabound2}
\end{eqnarray}
Consider that the condition of Lemma \ref{lem:pdf2} is satisfied for all $r\leq a$ if $f(\mathbf{x})>f_c$, follow the steps in Appendix \ref{sec:unbounded-l1}, \eqref{eq:masslb} holds, hence $P(B(\mathbf{a}_i,a))\geq 2k/N$. As long as \eqref{eq:proberr} holds, for sufficiently large $N$, $P(B(\mathbf{a}_i,\rho))<(k-1)/N+\Delta(N,k)<2k/N$. Therefore, $\rho<a$.

Then the bounds of $I_1$, $I_2$ and $I_3$ are the same as Appendix \ref{sec:bounded-unif}, except that \eqref{eq:b2} becomes
\begin{eqnarray}
&&\left|\frac{k-1}{NV(B(\mathbf{a}_i,\rho))}-\frac{k-1}{NP(B(\mathbf{a}_i,\rho))}f(\mathbf{a}_i)\right|\nonumber\\
&=&\frac{k-1}{NP(B(\mathbf{a}_i,\rho))}\left|\frac{P(B(\mathbf{a}_i,\rho))-f(\mathbf{a}_i)V(B(\mathbf{a}_i,\rho))}{V(B(\mathbf{a}_i,\rho))}\right|\nonumber\\
&\overset{(a)}{\leq}& \frac{k-1}{NP(B(\mathbf{a}_i,\rho))}C_c\rho^2 f(\mathbf{a}_i)\left(1+\ln\frac{1}{f(\mathbf{a}_i)}\right)\nonumber\\
&\overset{(b)}{\leq} & \frac{k-1}{NP(B(\mathbf{a}_i,\rho))}C_c\left(\frac{2P(B(\mathbf{a}_i,\rho))}{v_df(\mathbf{a}_i)}\right)^\frac{2}{d}f(\mathbf{a}_i)\left(1+\ln \frac{1}{f(\mathbf{a}_i)}\right)\nonumber\\
&\leq &\left\{
\begin{array}{ccc}
\left(\frac{2}{v_d}\right)^\frac{2}{d}C_c\frac{k-1}{N\left(\frac{k-1}{N}-\Delta(N,k)\right)^{1-\frac{2}{d}}} f^{1-\frac{2}{d}}(\mathbf{a}_i)\left(1+\ln \frac{1}{f(\mathbf{a}_i)}\right) &\text{if} & d\geq 2\\
\left(\frac{2}{v_d}\right)^2 C_c \frac{k-1}{N}\left(\frac{k-1}{N}+\Delta(N,k)\right)\frac{1}{f(\mathbf{a}_i)}\left(1+\ln \frac{1}{f(\mathbf{a}_i)}\right) &\text{if} & d=1
\end{array}
\right. \\
&\lesssim &\left\{
\begin{array}{ccc}
\left(\frac{k}{N}\right)^\frac{2}{d}\ln N &\text{if}& d\geq 2\\
\left(\frac{k}{N}\right)^2\ln N\frac{Na^d}{k} &\text{if} & d=1, 
\end{array}
\right.
\label{eq:ins-rate}
\end{eqnarray}

in which (a) comes from Lemma \ref{lem:pdf2}. (b) comes from \eqref{eq:masslb}.

Therefore, following the remaining steps in Appendix \ref{sec:bounded-unif}, we have
\begin{eqnarray}
\underset{\mathbf{x}\in S}{\sup}|\hat{f}(\mathbf{x})-f(\mathbf{x})|\lesssim \left\{
\begin{array}{ccc}
\left(\frac{k}{N}\right)^\frac{2}{d}+k^{-\frac{1}{2}}\sqrt{\ln \frac{N}{\epsilon}} &\text{if} & d\geq 2\\
\frac{k}{N} a^d\ln N+k^{-\frac{1}{2}}\sqrt{\ln \frac{N}{\epsilon}} &\text{if} & d=1.
\end{array}
\right. 
\end{eqnarray}

\textbf{Bound in $S^c$.} Define
\begin{eqnarray}
f_c=\frac{4k}{Nv_da^d}.
\end{eqnarray}

Recall the definition of $S_1$ in \eqref{eq:S1def}, for all $\mathbf{x}\notin S_1$, there exists a $\mathbf{x}'$ such that $\norm{\mathbf{x}'-\mathbf{x}}<h$ and $\mathbf{x}'\notin S$. Since $\mathbf{x}'\notin S$, $f(\mathbf{x}')\leq f_c$. Hence for all $\mathbf{x}\notin S_1$, 
\begin{eqnarray}
P(B(\mathbf{x},h))\leq f_+(\mathbf{x},h)V(B(\mathbf{x},h))\leq f_+(\mathbf{x}',2h)V(B(\mathbf{x},h))\leq 2f_c V(B(\mathbf{x},h)).
\end{eqnarray}
From \eqref{eq:h},
\begin{eqnarray}
2f_cV(B(\mathbf{x},h))\leq\frac{k}{2N}.
\end{eqnarray}

Define event $E_j$, such that $\mathbf{X}_j\notin S_1$ and $\rho_{k-1}(\mathbf{X}_j)<r_0$, in which $\rho_{k-1}(\mathbf{X}_j)$ is the $(k-1)$-th nearest neighbor distance of point $\mathbf{X}_j$, and $E=\cup_{j=1}^N E_j$. Then according to Chernoff inequality, 
\begin{eqnarray}
&&\text{P}(E_j)=\text{P}\left(\mathbf{X}_j\notin S_1, \rho_{k-1}(\mathbf{X}_j)<r_0\right)\nonumber
\\&\leq & \mathbb{E}\left[e^{-(N-1)P(B(\mathbf{x},r_0))}\left(\frac{e(N-1)P(B(\mathbf{X}_j,r_0))}{k-1}\right)^{k-1}\mathbf{1}(\mathbf{X}_j\notin S_1)\right]\nonumber\\
&\leq & e^{-\frac{1}{2}k}\left(\frac{1}{2}e\right)^k\nonumber\\
&=&e^{-\left(\ln 2-\frac{1}{2}\right)k}.
\end{eqnarray}
Hence
\begin{eqnarray}
\text{P}(E)=\text{P}\left(\cup_{j=1}^N E_j\right)\leq Ne^{-\left(\ln 2-\frac{1}{2}\right)k}.
\end{eqnarray}

If $k/\ln N\rightarrow \infty$, then for sufficiently large $N$, $\text{P}(E)<\epsilon/2$. The remaining proof assumes that $E$ does not happen. This condition holds with probability at least $1-\epsilon/2$. Then $\rho(\mathbf{X}_j)\geq h$ if $\mathbf{X}_j\notin S_1$. For all $\mathbf{x}\in S^c$, we have $\rho(\mathbf{x})\geq h/2$, because if $\rho(\mathbf{x})<h/2$, then there exists at least $k$ points in $B(\mathbf{x},h/2)$. According to the definition of $S$, $S_1$ and $S_2$, $B\left(\mathbf{x},h/2\right)\cap S_1=\emptyset$, thus $B(\mathbf{x},h/2)\subset S_1^c$. Therefore $\exists \mathbf{X}_j\in S_1^c$, and $\rho_{k-1}(\mathbf{X}_j)<h$, which contradicts with the assumption that $E$ does not happen. Therefore $\rho(\mathbf{x})\geq h/2$ holds for all $\mathbf{x}\in S^c$. Then
\begin{eqnarray}
V_0(B(\mathbf{x},h))\geq \frac{1}{2^d}v_dh^d,
\end{eqnarray}
and
\begin{eqnarray}
\hat{f}(\mathbf{x})\leq \frac{k-1}{NV_0(B(\mathbf{x},\rho(\mathbf{x})))}\leq \frac{k-1}{NV\left(B\left(\mathbf{x},\frac{1}{2}h\right)\right)}=\frac{2^d(k-1)}{Nv_dh^d}, \forall \mathbf{x}\in S^c.
\end{eqnarray}
From \eqref{eq:h},
\begin{eqnarray}
\hat{f}(\mathbf{x})\lesssim \frac{k}{Na^d}.
\label{eq:outs-rate}
\end{eqnarray}
From \eqref{eq:ins-rate} and \eqref{eq:outs-rate}, and note that $a\sim \ln N$ for sufficiently large $N$, with probability at least $1-\epsilon$, 
\begin{eqnarray}
\underset{\mathbf{x}}{\sup}|\hat{f}(\mathbf{x})-f(\mathbf{x})|\lesssim \left\{
\begin{array}{ccc}
\left(\frac{k}{N}\right)^\frac{2}{d}+k^{-\frac{1}{2}}\sqrt{\ln \frac{N}{\epsilon}} &\text{if} & d>2\\
\frac{k}{N}\ln^d N+k^{-\frac{1}{2}}\sqrt{\ln \frac{N}{\epsilon}} &\text{if} & d=1,2.
\end{array}
\right.
\end{eqnarray}
\section{Proof of Theorem \ref{thm:unbounded-unif-lb}}\label{sec:unbounded-unif-lb}

Define
\begin{eqnarray}
f_v(\mathbf{x}) = f_0(\mathbf{x})+vr^2 g\left(\frac{\mathbf{x}-\mathbf{a}_1}{r}\right)-vr^2 g\left(\frac{\mathbf{x}-\mathbf{a}_2}{r}\right),
\end{eqnarray}
in which $f_0$ is a fixed pdf, which ensures that $f_0(\mathbf{x})\geq m$ for $\mathbf{x}\in B(\mathbf{a}_1,r)\cap B(\mathbf{a}_2, r)$. $g(\mathbf{u})$ is an arbitrary function that supports on $B(\mathbf{0}, 1)$, has bounded Hessian and reaches its maximum $g_m$ at $\mathbf{u}=\mathbf{0}$. Then for any estimator $\hat{f}$,
\begin{eqnarray}
\underset{f\in \Sigma_C}{\sup} \mathbb{E}\left[\norm{\hat{f}-f}_\infty\right] &\geq & \underset{v\in \{-1,1\}}{\sup}\mathbb{E}\left[\norm{\hat{f}-f_v}_\infty\right]\nonumber\\
&\geq & \mathbb{E}\left[\norm{\hat{f}-f_V}_\infty\right]\nonumber\\
&\geq & \frac{1}{4}\norm{f_{v_1}-f_{v_2}}_\infty e^{-ND(f_{v_1}||f_{v_2})}\nonumber\\
&\geq & r^2 e^{-Nr^{d+4}}.
\end{eqnarray}
Let $r\sim N^{-1/(d+4)}$, then
\begin{eqnarray}
\underset{f\in \Sigma_C}{\sup} \mathbb{E}\left[\norm{\hat{f}-f}_\infty \right]\gtrsim N^{-\frac{2}{d+4}}.
\end{eqnarray}

\small \bibliography{macros,knn}
\bibliographystyle{ieeetran}
\end{document}